\documentclass[11pt]{article}
\usepackage{bm}
\usepackage{graphicx}
\usepackage{amsfonts}
\usepackage{amsmath}
\usepackage{amssymb}
\usepackage{tikz}
\usetikzlibrary{arrows.meta}
\usepackage{url}
\usepackage{fancyhdr}
\usepackage{dsfont}
\usepackage{indentfirst}
\usepackage{enumerate}
\usepackage{amsthm}
\usepackage{color}
\usepackage{natbib}
\usepackage[colorlinks=true,citecolor=blue]{hyperref}
\usepackage[british]{babel}
\def\d{\mathrm{d}}

\newcommand{\ES}{\mathrm{ES}}
\newcommand{\E}{\mathbb{E}}
\newcommand{\R}{\mathbb{R}}

\newcommand{\id}{\mathds{1}}
\newcommand{\lcx}{\preceq_{\mathrm{cx}}}

\newcommand{\X}{\mathcal X}

\newcommand{\M}{\mathcal M}

\newcommand{\esssup}{\mathrm{ess\mbox{-}sup}}

\renewcommand{\ge}{\geqslant}
\renewcommand{\le}{\leqslant}
\renewcommand{\geq}{\geqslant}
\renewcommand{\leq}{\leqslant}
\renewcommand{\epsilon}{\varepsilon}
\theoremstyle{plain}
\newtheorem{theorem}{Theorem}

\newtheorem{lemma}{Lemma}
\newtheorem{proposition}{Proposition}
\theoremstyle{definition}
\newtheorem{definition}{Definition}
\newtheorem{example}{Example}

\theoremstyle{remark}
\newtheorem{remark}{Remark}
\theoremstyle{definition}

\renewcommand{\cite}{\citet}

\DeclareMathOperator*{\argmin}{arg\,min}
\DeclareMathOperator*{\argmax}{arg\,max}
\usepackage[misc]{ifsym}

\usepackage[onehalfspacing]{setspace}
\begin{document}

 \title{Choquet regularization for reinforcement learning}

 \author{Xia Han\thanks{Department of Statistics and Actuarial Science, University of Waterloo, Canada. Email: x235han@uwaterloo.ca}
 \and Ruodu Wang\thanks{Department of Statistics and Actuarial Science, University of Waterloo, Canada. Email: wang@uwaterloo.ca}\and Xun Yu Zhou\thanks{ Department of IEOR, Columbia University,  USA. Email: xz2574@columbia.edu.}
 }
 \maketitle

 \centerline{\Large {\bf Abstract}}
We propose \emph{Choquet regularizers} to measure and manage the level of exploration for reinforcement learning (RL), and reformulate the continuous-time entropy-regularized RL problem of  \cite{WZZ20a} in which we replace the differential entropy used for regularization with a Choquet regularizer. We derive the Hamilton--Jacobi--Bellman equation of the problem, and solve it explicitly   in the linear--quadratic (LQ) case via maximizing statically a mean--variance constrained Choquet regularizer. Under the LQ setting, we derive explicit optimal distributions  for  several specific Choquet regularizers, and conversely identify the Choquet regularizers that generate a number of broadly used exploratory samplers such as $\epsilon$-greedy, exponential, uniform and Gaussian.

%\vspace{0.5cm}

\vspace{1cm}

\noindent{\textsc{Keywords}: Reinforcement learning,  Choquet integrals, continuous time, exploration, regularizers, quantile,  linear--quadratic control}.

 \section{Introduction}

Reinforcement learning (RL) is one of the most active and fast developing subareas in  machine learning. The foundation of RL is ``trial and error" -- to {\it strategically} explore different action plans  in order to find the best plan as efficiently and economically as possible. A key question to this inherent exploratory approach for RL is to seek a proper tradeoff between exploration and exploitation, for which one needs to first quantify the level of exploration. Because exploration is typically captured by randomization in the RL study, entropy has been employed to measure the magnitude of the randomness  and hence that of the exploration  -- a uniform distribution representing a completely blind search has the maximum entropy while a Dirac mass signifying no exploration at all has the minimum entropy. Discrete-time entropy-regularized (or “softmax”) RL formulation has been proposed which introduces a weighted entropy value of the  exploration as a regularization term into the objective function (\citealt{ZMBD08,NNXS17,HZAL18}).
%
%
%It has been widely applied to other fields including quantitative finance, in e.g.,
%algorithmic and high frequency trading, smart order routing, and portfolio management. RL is able to address many problems that human traders and portfolio managers could not solve in the past time by providing  sufficient amount of microstructure data for training and adaptive learning.  Exploration and exploitation are the essential to RL. An agent does not pre-specify a structural model or a family of models but, instead, gradually learns the best (or near-best) strategies  via trial and error, through interactions with an environment and possibly with other agents.
%Agents can  improve from current sub-optimal solutions to the ultimate global optimal one by exploration,  but it is inherently costly in terms of resource, time and opportunity.  Meanwhile, pure exploitation, i.e., myopically picking the current solution based solely on past experience, though easy to implement, tends to yield sub-optimal global solutions. Therefore, finding an appropriate trade-off between exploration-exploitation solution is crucial for RL algorithm design to improve the learning and the optimization procedure.
For continuous-time RL, \cite{WZZ20a} formulate an entropy-regularized,
distribution-valued stochastic control problem for diffusion processes, and derive
theoretically the Gibbs (or Boltzmann) measure as the optimal distribution for exploration which specializes to Gaussian when the problem is linear--quadratic (LQ).
  \cite{WZ20} and \cite{GXX22} apply the results of \cite{WZZ20a} to  a continuous-time entropy-regularized Markowitz's
mean--variance portfolio selection problem and a Langevin diffusion for simulated annealing, respectively.
\cite{GXZ20} analyze both quantitatively and qualitatively the impact of entropy regularization for mean-field games with learning in a finite time horizon.  There have been recently many other developments along  this direction of RL in continuous time; see \cite{TZZ21, MZZ21, JZ21, JZ22}  and the references therein.
% In their work, the  Shanon's differential entropy is used to measure the level of exploration and the objective is   to maximize  the entropy-regularized reward function  when the mean and  the variance of  distribution  are known  They show that, if, at any given state, the agent sets out to engage in exploration, then she needs look no further than Gaussian distributions,% which implies that   the optimal distribution for Shanon's differential entropy is   Gaussian distribution.

 While the  entropy is a reasonable   metric  to quantify the information gain of exploring the environment and entropy regularization can indeed explain some broadly used exploration distributions such as Gaussian, there are two closely related  open  questions:
\begin{enumerate}
\item Distributions other than Gaussian, such as exponential   or uniform,   are also widely used for exploration in RL. What regularizer(s) can theoretically justify the use of a given  class of exploratory distributions?
\item  What are the optimal exploratory distributions for regularizers other than the entropy?
\end{enumerate}

%\cite{WZZ20} obtained Gaussian distribution as the optimal distribution in a LQ control problem.

  In this paper, we study these two questions in the setting of continuous-time diffusion processes,
  % To the best of our knowledge,  few   consider the LQ optimization problems where the regularization is defined by the other terms  than DE. Therefore,
by introducing   a new class of regularizers borrowing from the literature of risk metrics. % to balance   exploration-exploitation, thereby establishing a link between the rich literature on risk metrics and RL.
Risk metrics, roughly speaking, include risk measures and variability measures, which are two separate and active research  streams  in the general area of risk management. %,  popular in banking and finance for various purposes, such as calculating solvency capital reserves, pricing of insurance risks, performance analysis,   internal risk management, and optimization.
Value-at-risk (VaR), expected shortfall (ES) and various coherent or convex
risk measures, introduced by \cite{ADEH99},  \cite{FS02a}, and \cite{D02},  are popular examples of
risk measures. Variance, the Gini deviation, interquantile range and deviation measures of \cite{RUZ06} are instances of variability measures.
 There has been a rich body of study  on risk metrics in the past two decades; see \cite{FS16} and the references therein.   %In the practice of risk measurement, one very often assesses a risk through its distribution, which is obtained via statistical and simulation analysis and corresponds to a well-known topic of optimizing  the risk measures with distributional uncertainty.  %General speaking,  distributional  uncertainty  means that  the probability distribution of uncertain outcomes is unknown or cannot be uniquely identified. This is due to the fact that one usually has limited information to infer the ``true" distribution, which is typically the case in most applications where only sample data is available. How to account for distributional ambiguity in decision making has  become a popular topic in the field of  economics, finance, control system, and operations research.  A  common method to describe the set of alternative distributions, the uncertainty set, is  characterized by moment constraints, see for example,  \cite{P07}, \cite{BDNT10},  \cite{W14},  \cite{CLM20}, and \cite{LCLW20}.

%We extend the work of \cite{WZZ20a} and propose  a  novel class of regularizers.   %and use it to  determine  the optimal   distributions   via maximizing the  regularizers with mean and variance constraints.  In other words,

We introduce what we call   \emph{Choquet regularizers}, %  based on  their various    theoretical and practical advantages.
 which  belong to the class of the signed Choquet integrals recently studied by \cite{WWW20c} in the context of risk management. A signed Choquet integral in general gives rise to a nonlinear and non-monotone expectation in which the state of nature is weighted by a probability weighting or distortion function in calculating the expectation. It includes as special cases Yaari's dual utility (\citealp{Y87}) and distortion risk measures (\citealp{K01} and \citealp{A02}), which are commonly used monotone functionals,
 and appears in   rank-dependent utility (RDU) theory; see  \cite{Q82}, \cite{GS89}, \cite{TK92} and \cite{DW01}  in the related literature of behavioral economic  theory.

 There are several reasons to use Choquet regularizers for RL due to a number of     theoretical and practical advantages. First, they satisfy several ``good" properties such as quantile additivity, normalization, concavity, and consistency with convex order (mean-preserving spreads) that facilitate analysis as regularizers.
Second, Choquet regularizers are non-monotone. In order to measure exploration, %which is closely related  to variability of the control variable,
monotonicity is irrelevant, in contrast to assessing risk or wealth. For instance, a degenerate distribution should be associated with no-exploration regardless of its position, in which case non-monotone mappings  should be used.
 %Since monotonicity does not apply to all practical measures of variability, see for instance, variance, standard deviation, or deviation measures in \cite{R13} and  \cite{GMZ09}, .
%There are many preferences used in practice that are not monotone, including  the mean-variance and the mean-standard-deviation criteria of \cite{M52}.
%Moreover,
%Choquet integrals also appear naturally in   rank-dependent utility (RDU) theory; see  \cite{Q82}, \cite{GS89}, \cite{TK92} and \cite{DW01}  in the related literature of behavioral economic  theory.
Moreover, the use of Choquet regularizers is closely  connected to  distributionally robust optimization (DRO) where a Wasserstein distance naturally induces a special class of Choquet regularizers, whereas DRO itself is an important approach for learning and for correcting the inherent flaws suffered by classical model-based estimation and optimization.
Finally, as we will see later in the paper, for any given location--scale class of distributions, there exists a common Choquet regularizer such that the corresponding regularized continuous-time LQ control for RL has optimal distributions in that class.

We take the same continuous-time exploratory stochastic control problem as in \cite{WZZ20a}, except that the entropy regularizer is replaced with a Choquet regularizer.
In the general case we derive the Hamilton-Jacobi-Bellman (HJB) equation. However, in sharp contrast to \cite{WZZ20a}, solving the HJB equation and thus obtaining the optimal distributional policies via verification theorem remain a significant  open question. To obtain explicit solutions, we focus on the LQ case. The form of the LQ-specialized HJB equation suggests that the problem boils down to a static optimization in which the given Choquet regularizer is to be maximized over distributions with given mean and variance. It turns out this last problem has been solved explicitly by \cite{LCLW20}. The optimal distributions form a location--scale family, whose shape depends on the choices of the Choquet regularizer. The solutions to the static problem are then employed to solve the original LQ-based HJB equation explicitly and to derive the optimal samplers for exploration under the given  Choquet regularizer. As expected, optimal distributions are no  longer necessarily Gaussian as in \cite{WZZ20a}, and are now dictated by  the choice of Choquet regularizers. However, the following feature of the entropy-regularized solutions revealed in \cite{WZZ20a} remains intact: the means of the optimal distributions are linear in the current state and  independent of the exploration, whereas the variances are determined by the exploration but irrespective of the current state. Along an opposite line of inquiry, we are able to identify a proper Choquet regularizer in order to interpret a given exploratory distribution. Specifically, we
derive  the regularizers that generate some common exploration measures including $\epsilon$-greedy, three-point, exponential, uniform and Gaussian.

The rest of the  paper is organized as follows.
We introduce Choquet regularizers in Section \ref{sec:2}, and present their basic properties as well as an axiomatic characterization based on existing results of \cite{WWW20b,WWW20c}.  In Section \ref{sec:3},   we  formulate the continuous-time Choquet-regularized   RL control problem and derive the  HJB equation.  We then motivate a mean--variance constrained Choquet regularizer maximization problem for LQ   control. This problem is studied in details in Section \ref{sec:4}, including discussions on some  special  regularizers  arising  from problems in finance,  optimization, and risk management.
In Section \ref{sec:6}, we return to the exploratory LQ control problem and solve it completely. We also present examples linking specific exploratory distributions with the corresponding Choquet regularizers.
Finally,
Section \ref{sec:7} concludes the paper.
%Throughout the main part of the paper, we focus on distributions on $\R$, because Choquet regularizers are the most naturally defined on $\R$ and this allows us to work with quantiles; we comment on generalizations to $\R^d$ in Section \ref{sec:7}.

 \section{Choquet regularizers}
 \label{sec:2}
Throughout the paper,  $(\Omega, \mathcal{F}, \mathbb{P})$ is an atomless probability space. With a slight abuse of notation,
let %the set $\mathcal{X}$ be a  convex cone of random variables and
$\mathcal M$  denote both the set of (probability) distribution functions of real random variables and the set of Borel probability measures on $\mathbb{R}$, with  %in  $\mathcal{X}$.
the obvious identity $\Pi(x)\equiv  \Pi((-\infty, x])$   for $x \in \mathbb{R}$
and
$\Pi \in \mathcal{M}$. We denote by  $\mathcal M^p$,  $p\in[1,\infty)$,  the set of distribution functions or probability measures with finite $p$-th moment.   For a random variable $X$ and a distribution $\Pi$,  we write $X \sim \Pi$ if the distribution of  $X$ is $\Pi$ under $\mathbb{P}$, and $X \stackrel{\rm d}{=} Y$ if two random variables $X$ and $Y$ have the same distribution.
 We denote by $\mu$ and $\sigma^2$  the mean and  variance functional on  $\mathcal{M}^2$, respectively; that is, $\mu(\Pi)$ is the mean of $\Pi$ and $\sigma^2(\Pi)$  the variance of $\Pi$ for $\Pi \in\mathcal M^2$. %For an absolutely continuous $\Pi$, we denote by  $\Pi$ the probability density function and $\mathcal{P}$ the set of density functions.

Given a function  $h:[0,1]\to \R$  of bounded variation with $h(0)=0$, % and  $\Pi\in \mathcal M$,
  the functional $I_h$  on $\mathcal M$ is defined as % is defined as $I_h(\Pi):=\int  h\circ \Pi([x,\infty))\d x$, in which
\begin{equation}
I_h(\Pi)\equiv \int  h\circ \Pi([x,\infty))\d x:=\int_{-\infty}^{0}\left[h\circ \Pi ( [x,\infty) )-h(1)\right]\d x+\int_0^\infty h\circ \Pi([x,\infty))\d x, \;\Pi\in \mathcal M,\label{eq:def}
\end{equation}
 assuming that \eqref{eq:def} is well defined (which could take the value $\infty$).
   The function $h$ is called a \emph{distortion function}, and the functional $I_h$ is called a \emph{signed Choquet integral} by \cite{WWW20c}. If $h(x)\equiv x$
   then $I_h$ reduces to the mean functional; thus, $I_h$ is a nonlinear generalization of the mean/expecation.
 If $h$ is increasing and satisfies $h(0)=1-h(1)=0$, then $I_h$  is called an  \emph{increasing Choquet integral}, which include as special cases the two most important risk measures used in current banking and insurance regulation, VaR and ES.\footnote{This functional $I_h$ is termed differently in different fields. For example, it is known as Yaari's dual utility (\citealp{Y87}) in decision theory,  distorted premium principles (\citealp{D94} and \citealp{WYP97}) in insurance and distortion risk measures  (\citealp{K01} and \citealp{A02}) in finance.}  %Compared with  an increasing Choquet integral, a signed  one is  more general and is generated by a function $h$ that is not necessarily monotone.
  %In this work, we are particularly interested in signed Choquet integrals for various practical and theoretical reasons as mentioned in \cite{WWW20}.   To point out some of them, monotonicity does not apply to all practical measures of variability, see for instance, variance, standard deviation, or deviation measures in \cite{R13} and  \cite{G09}. In other words, it is necessary to consider signed Choquet integrals  as long as the  measure of variability is involved.   From the view of  economic decision theory, signed Choquet integrals appear naturally in much rank-based decision making; see \cite{Q82}, \cite{GS89} and \cite{DW01}.   Besides,  there are many preferences used in practice that are not monotone. A prominent example is the mean-variance and the mean-standard-deviation preferences as already studied by \cite{M52}.

  Next, we  define the  {\it Choquet regularizer}, a main object of this paper.
We are particularly interested in a subclass of signed Choquet integrals, where $h$ satisfies the following properties:
\begin{enumerate}[(i)]
\item $h$ is concave;
\item $h(1)=h(0) = 0$.
%\item $h\ge 0$.
\end{enumerate}
Let us briefly explain the interpretations and implications of the above two conditions. %\com{I now write $h\ge 0$ as a consequence, not a condition.}
Condition (i) is equivalent to several other properties, and in particular,
to that $I_h$ is a concave mapping and to that $I_h$ is consistent with \emph{convex order};\footnote{Let $\Pi_1$ and $\Pi_2$ be two distribution functions with  finite means. Then, $\Pi_1$ is smaller than $\Pi_2$  in \emph{convex order}, denoted by  $ \Pi_1\lcx \Pi_2$, if   $\mathbb{E}[f( \Pi_1)] \leq \mathbb{E}[f( \Pi_2)]$ for all convex functions $f$,  provided that the two expectations exist. It is immediate that $\Pi_1 \lcx \Pi_2$ implies $\mathbb{E}[\Pi_1]=\mathbb{E}[\Pi_2]$.}
 see Theorem 3 of \cite{WWW20c} for this  claim and several other equivalent properties.
Here, concavity of $I_h$ means
$$I_h(\lambda \Pi_1 + (1-\lambda) \Pi_2 ) \ge \lambda I_h( \Pi_1)  + (1-\lambda) I_h(\Pi_2 ), ~~~~\mbox{for all~}\Pi_1,\Pi_2\in \mathcal M\mbox{~and~}\lambda\in [0,1],$$
and consistency with convex order means
$$I_h( \Pi_1  ) \le  I_h(\Pi_2 ), ~~~~\mbox{for all~}\Pi_1,\Pi_2\in \mathcal M \mbox{~with~} \Pi_1\lcx \Pi_2.$$
If $\Pi_1\lcx \Pi_2$, then  $\Pi_2$ is also called a  {\it mean-preserving spread} of $\Pi_1$ (\citealp{RS70}), which intuitively means that $\Pi_2$ is more spread-out (and hence ``more random")  than $\Pi_1$.
The above two properties do indeed suggest  that $I_h(\Pi)$ serves as a measure of randomness for  $\Pi$, since both a mixture and a mean-preserving spread introduce extra randomness; see e.g., \cite{AS13} for related discussions.
%\comb{XYZ: But if two distribution do not have the same  mean, can we still compare their randomness via $I_h$?}
%\com{$I_h$ is location invariant by (ii). I extended a discussion on these useful properties below}
Condition (ii), on the other hand,  is equivalent to $I_h(\delta_c)=0$ $\forall c\in \R$, where $\delta_c$ is the Dirac mass at $c$. That is, degenerate distributions do not have any randomness measured by $I_h$.

%Next,   we formally define  {Choquet regularizers}
%a signed Choquet integrals.

\begin{definition}\label{def:CE}Let  $\mathcal H$ be the set of $h:[0,1]\to \R$ satisfying (i)-(ii).   A functional $\Phi:  \mathcal M\to (-\infty,\infty]$ is a \emph{Choquet regularizer}  if there exists  $h\in \mathcal H$ such that $\Phi=I_h$, that is, \begin{equation}
\Phi(\Pi)= \int_\R h\circ \Pi([x,\infty))\d x, \label{eq:def2}
\end{equation}
and this  $\Phi$  will henceforth be denoted by $\Phi_h$.

\end{definition}

To clarify  on notation, we require $h\in \mathcal H$ for $\Phi_h$,
while there is no such requirement for $I_h$. Moreover, we  call $\Phi_h$ to be location invariant and scale homogeneous if
  $\Phi_h(\Pi')=\lambda \Phi_h (\Pi)$
 where $\Pi' $ is the distribution of $\lambda X+c$ for $\lambda >0$, $c\in \R$ and $X\sim \Pi$.

We summarize some useful properties of $\Phi_h$   in the following lemma.
\begin{lemma}\label{lem:propety}
For $h\in \mathcal H$,  $\Phi_h$  is  well defined, non-negative,  and
 location invariant and scale homogeneous.
\end{lemma}
\begin{proof}
First, a concave $h$ with $h(0)=h(1)$ has to be first increasing and then decreasing on $[0,1]$. Hence $h$ has bounded variation, and the two integrals in \eqref{eq:def} are well defined. Moreover,   (i) and (ii) imply $h\ge 0$, which further yields  that both terms in \eqref{eq:def} are non-negative. So $\Phi_h$ is  well defined and non-negative.
Location invariance and scale homogeneity follow from Proposition 2 (iii) and (iv) of \cite{WWW20b}.
\end{proof}
Each property in Lemma \ref{lem:propety} has a simple interpretation  for a regularizer that measures  the level of  randomness, or the level of exploration in the RL context of this paper.
\begin{enumerate}[(a)]
\item Well-posedness: Any distribution for exploration can be measured.\footnote{This property is technically important    since functional properties of $I_h$ could be very difficult to analyze if one  faces a quantity such as $\infty-\infty$. As an example, consider  $h(x)=x$  leading to $I_h$ being the mean functional. In this case, $I_h$  is only well defined on some subsets of $\mathcal M$.}
\item  Non-negativity: Randomness is measured in non-negative values.
\item Location invariance:  The measurement of randomness does not depend on the location.
\item Scale homogeneity: The measurement of randomness is linear in its scale.

\end{enumerate} For a distribution $\Pi\in\mathcal{M}$, let its   left-quantile for $p\in(0,1]$ be defined as  $$Q_\Pi(p)=\inf \{x\in \R: \Pi(x) \ge p\}, \label{eq:l_q}$$ whereas its right-quantile function  for $p\in[0,1)$ be defined as
$$
Q^+_{\Pi}(p)=\inf \{x\in \R: \Pi(  x) > p \}.
$$
It is useful to note that $\Phi_h$ admits a quantile representation as follows; see Lemma 1 of \cite{WWW20b}. %\com{Added more on the quantile formulation.}
 \begin{lemma}\label{lem:qr}
For $h\in \mathcal H$  and  $\Pi\in \mathcal M$,
\begin{itemize}
\item[(i)] if $h$ is right-continuous, then $\Phi_h(\Pi)=\int_{0}^{1} Q^+_{\Pi}(1-p) \d  h(p)$;
\item[(ii)]  if $h$ is left-continuous, then $\Phi_h(\Pi) =\int_0^1  Q_\Pi(1-p) \d  h(p)$;
\item[(iii)] if $Q_\Pi$ is continuous, then $\Phi_h(\Pi) =\int_0^1  Q_\Pi(1-p) \d  h(p)$.
\end{itemize}
\end{lemma}

Choquet regularizers include, for instance,  range,  mean-median deviation, the Gini deviation, and  inter-ES differences. Moreover,   standard deviation can be written as the supremum of Choquet regularizers; see Examples 1, 3 and 4 of \cite{WWW20c}.  Variance also has a related representation (Example 2.2 of \citealp{LCLW20}): $$\sigma^2 (\Pi)= \sup_{h\in \mathcal H} \left\{\Phi_h(\Pi)-\frac 14||h'||_2^2\right\},~~~ \Pi\in\mathcal M,$$
where $||h'||_2^2=\int _0^1 (h'(p))^2 \d p$ if $h$ is continuous with  a.e.~right-derivative $h'$, and $||h'||_2^2:=\infty$ if $h$ is not continuous.

Concave signed Choquet integrals are characterized by, e.g.,~\cite{WWW20c}, which is essentially a consequence of the seminal work of \cite{Y87} and \cite{S89}; see also Theorem \ref{thm:charac} below.
In what follows, we say that
$\Phi=\Phi_h$ is \emph{quantile additive} %\comb{XYZ: Can you give some concrete examples here?}%\com{All $\Phi_h$ are quantile additive as in Proposition 1.}
if  for all $\Pi_1,\Pi_2\in \mathcal M$,
$\Phi(\Pi_1\oplus \Pi_2)=\Phi(\Pi_1) + \Phi(\Pi_2)$
where the quantile function of $\Pi_1\oplus \Pi_2$ is the sum of those of $\Pi_1$ and $\Pi_2$. In other words,
$Q_{\Pi_1\oplus \Pi_2}  =Q_{\Pi_1} + Q_{\Pi_2}.$  Moreover, we say that
$\Phi$ is  \emph{continuous at infinity} if $\lim _{M \rightarrow 1} \Phi((\Pi\wedge M)\vee (1-M))=\Phi(\Pi)$,  and $\Phi$
is  \emph{uniform sup-continuity} if for any $\epsilon>0$, there exists $\delta>0$, such that $| \Phi(\Pi_1)- \Phi(\Pi_2) |<\epsilon$ whenever $\esssup |\Pi_1-\Pi_2|<\delta$, where $\esssup$ is the essential supremum defined by $\Pi^{-1}(1)$ .

We give the following simple characterization for our Choquet regularizers based on Theorems 1 and 3 of \cite{WWW20b}.%\com{I did not write down the definition of the two continuities. I may be able to reduce the continuity requirement a bit.}
% The following characterization may require some continuity conditions, since
%results from  \cite{WWW20} are stated for bounded random variables.
%This is not the focus of the note, and I think it can be easily fixed.
 \begin{theorem}\label{thm:charac}
 A functional $\Phi_h$ is a Choquet regularizer in \eqref{eq:def2}  if and only if it satisfies  all of the following properties
 \begin{enumerate}[(i)]
 \item $\Phi_h$ is quantile additive;
 \item $\Phi_h$ is concave or $\lcx$-consistent;
 \item $\Phi_h \ge 0$ and $\Phi_h(\delta_c)=0$ for all $c\in \R$;
 \item $\Phi_h$ is continuous at infinity
and uniformly sup-continuous.
 \end{enumerate}
 \end{theorem}

Note that Theorems 1 and 3 of \cite{WWW20b} are stated in terms of a risk measure defined on the space of real random variables, say
$\mathcal{X}$, while here $\Phi_h$ is defined on  $\mathcal{M}$. To use these results, we can define  ${\rho}: \mathcal{X} \rightarrow \mathbb{R}$ by ${\rho}(X)=\Phi_h(\Pi)$ where $X \sim \Pi$, which is automatically law-invariant.\footnote{Law-invariance means that $\rho(X)=\rho(Y)$ for  $X \stackrel{\rm d}{=} Y$.}  On the other hand, Theorem 1 in \cite{WWW20b} requires
an extra continuity condition  to imply that  $h$ has bounded variation on $[0, 1]$, which is guaranteed here by condition (iii).    In fact,   condition (i) is equivalent to comonotonic additivity of ${\rho}$.\footnote{A random vector $(X_1,\dots,X_n)$ is called \emph{comonotonic} if there exists a random variable $Z\in\X$ and increasing functions $f_1,\dots,f_n$ on $\R$ such that $X_i=f_i(Z)$ almost surely for all $i=1,\dots,n$.
Comonotoic-additivity means that  $\rho(X+Y)=\rho(X)+\rho(Y)$ if $X$ and $Y$ are comonotonic.}   Continuity  at  infinity  and uniform sup-continuity of ${\rho}$ can be defined in paralell to those of $\Phi_h $.  Finally,   $h(1)=h(0) = 0$ is equivalent to $\Phi_h(\delta_c)=0$ for all $c\in \R$. Theorem \ref{thm:charac} hence follows directly from  Theorems 1 and 3 of \cite{WWW20b}.

\begin{remark}
If $h $ is not constantly $0$,
 Choquet regularizers   belong to the class of \emph{generalized deviation measures} in \cite{RUZ06} and   \cite{GMZ09}.
Moreover,
 Choquet regularizers can be used to construct law-invariant generalized deviation measures. Indeed,
 combining characterization of generalized deviation measures    in Proposition 2.2 of \cite{GMZ09} and the  quantile representation of signed Choquet integrals in  Lemma \ref{lem:qr}, all law-invariant generalized deviation measures  can be represented  as a supremum of some Choquet regularizers of the type \eqref{eq:def2}.  This includes standard deviation and mean absolute deviation as special cases.
%\com{In my personal opinion, the notion of generalized deviation measures, which require subadditivity, is less relevant  than   concavity or $\lmps$-consistency in   our setting.}
\end{remark}

We conclude this section by comparing the Choquet regularization with the differential entropy regularization, the latter having been used for exploration--exploitation balance in  RL; see  \cite{WZZ20a, WZ20,GXZ20}.
For an absolutely continuous $\Pi$, we  define  DE,  Shannon's differential entropy,  as \begin{equation}\label{eq:DE}{\rm DE}(\Pi):=-\int_{\R}\Pi'(x)\log(\Pi'(x))\d x.\end{equation} %If $Q_\Pi$ in \eqref{eq:l_q} is differentiable,
\cite{SS12} show that  \eqref{eq:DE} admits a different quantile representation % \com{Added DE}
\begin{equation}\label{eq:quan}
{\rm DE}(\Pi)= \int_0^1 \log (Q'_\Pi(p)) \d p.
\end{equation}
    %DE is commonly regarded as a  natural choice for entropy regulation  to balance the exploration-exploitation in   RL. %see e.g., \cite{WZZ20a}, \cite{WZ20} and \cite{GXZ20}.
 It is clear that DE is location invariant, but not scale homogeneous. It is not quantile additive either. Therefore,  DE is {\it not} a Choquet regularizer.
 %, as quantile additivity essentially characterizes signed Choquet integrals.

 \section{Exploratory control with Choquet regularizers}\label{sec:3}

%\subsection{Exploratory formulation in continuous time and spaces} Now,  based on the result in Lemma \ref{lem:liu},

In this section, we   first  introduce an  exploratory stochastic control problem for RL in continuous time and spaces which  was originally proposed in \cite{WZZ20a}, and then reformulate it with Choquet regularizers. %The model formulation is given as follows.

Let $ \mathbb{F}= \{\mathcal{F}_t\}_{t \geq 0}$  be a filtration   defined on $(\Omega, \mathcal{F}, \mathbb{P})$ along with an $\{\mathcal{F}_t\}_{t \geq 0}$-adapted Brownian motion $W=\{W_t\}_{t\geq0}$,
the   filtered probability space  satisfying the usual assumptions of completeness and right continuity.  All stochastic processes introduced below are supposed to be adapted processes in this space.

The classical stochastic control problem is to control the state  dynamic described by a stochastic differential equation (SDE)
\begin{equation}\label{eq:cl_sto}
\d X_{t}^{u}=b\left(X_{t}^{u}, u_{t}\right) \d t+\xi\left(X_{t}^{u}, u_{t}\right) \d W_{t},~ t>0 ; \quad X_{0}^{u}=x \in \mathbb{R},
\end{equation}
where $u=\{u_{t}\}_{ t \geq 0}$ is the control process taking value in a given action space $U$. Throughout this paper, for ease of notation we assume that the state and Brownian motion are  scalar-valued processes. Moreover, we suppose that the control is also one-dimensional, which is however an essential assumption  because the Choquet regularizer to be involved is defined only for distributions on $\R$.\footnote{See Section \ref{sec:7} for a discussion about how we may extend the notion of Choquet regularizer to multi-dimensions.} %$x$ is a generic variable representing a current state of the system dynamics and

Similarly to  \cite{WZZ20a}, we  give the ``exploratory" version of the state dynamic \eqref{eq:cl_sto} when the control is randomized,  motivated by repetitive learning in RL. The control process is now randomized, leading to a distributional or exploratory control process $\Pi=\{\Pi_{t}\}_{ t \geq 0}$, where $\Pi_{t}\in \mathcal M(U)$ is the probability distribution function for control at time $t$,  %whose density function is denoted $\Pi=\{\Pi_{t}\}_{ t \geq 0}$.
with  $\mathcal{M}(U)$  being the set of   distribution functions on $U$.  %and $\mathcal{P}(U)$ be the sets of   probability measures  and density functions on $U$, respectively.
%In this paper, we set $U=\R$ unless otherwise specified; in some examples we use $U=\R_+:=[0,\infty)$.
For a given such distributional control  $\Pi$,  the exploratory version of the state dynamic is
\begin{equation}\label{eq:dyn}
\d X_{t}^{\Pi}=\tilde{b}\left(X_{t}^{\Pi}, \Pi_{t}\right) \d t+\tilde{\xi}\left(X_{t}^{\Pi}, \Pi_{t}\right) \d W_{t},~ t>0 ; \quad X_{0}^{\Pi}=x \in \mathbb{R},
\end{equation} where
 the coefficients $\tilde{b}(\cdot, \cdot)$ and $\tilde{\xi}(\cdot, \cdot)$ are defined as
\begin{equation}\label{eq:tilde_b}
\tilde{b}(y, \Pi)=\int_{U} b(y, u)  \d \Pi(u), \quad y \in \mathbb{R},~ \Pi \in \mathcal{M}(U),
\end{equation}
and
\begin{equation}\label{eq:tilde_xi}
\tilde{\xi}(y, \Pi)=\sqrt{\int_{U} \xi^{2}(y, u)   \d \Pi(u)}, \quad y \in \mathbb{R}, ~\Pi \in \mathcal{M}(U).
\end{equation}
The ``exploratory state process" $\{X_{t}^{\Pi}\}_{t\geq0}$ describes the average of
the state processes under (infinitely) many different classical control processes sampled from the exploratory control $\Pi=\{\Pi_{t}\}_{ t \geq 0}$.

Define   the exploratory reward as
\begin{equation}\label{eq:tilde_r}
\tilde{r}(y, \Pi)=\int_{U} r(y, u)  \d \Pi(u), \quad y \in \mathbb{R},~ \Pi \in \mathcal{M}(U),
\end{equation} where $r$ is the reward function.  Further, we use a Choquet regularizer  $\Phi_h$ to measure the level of exploration, and  the aim of the exploratory control is to achieve the maximum expected total
discounted and regularized exploratory  reward   represented by the optimal value function \begin{align}\label{eq:rl}
 V(x)= \sup_{\Pi \in \mathcal A(x)} \E_x \left[ \int_0^\infty e^{-\rho t} \left( \tilde{r}(X_t^{\Pi}, u)  + \lambda \Phi_h(\Pi)
\right)\d t \right],
 \end{align}
  where  $\rho>0$ is a given discount rate,  $\lambda>0$ is the {\it temperature} parameter representing the weight being put on exploration, $\mathcal{A}(x)$ is the set of  admissible distributional controls (which may in general depend on $x$), and  $\E_x$ represents the conditional expectation  given  $X^\Pi_0 = x$.
  %Here, we do not distinguish between using $\Phi_h(\Pi)$ and  $\Phi_h(\Pi)$.

The precise definition of $\mathcal{A}(x)$ depends on the specific dynamic model under consideration and the specific problems one wants to solve, which may vary from case to case.
We will define $\mathcal{A}(x)$ precisely later for the   linear--quadratic (LQ) control case, which will be the main focus of the paper. Note that \eqref{eq:rl}  is called a regularized relaxed stochastic control problem in \cite{WZZ20a}, and    we refer to    \cite{WZZ20a} and \cite{WZ20}   for more details on the motivation of \eqref{eq:dyn}-\eqref{eq:rl}.

Controls in $\mathcal{A}(x)$ are measure (distribution function)-valued stochastic adapted processes, which are open-loop controls in the control terminology. A more important notion in RL is the feedback  (control) {\it policy}. Specifically, a deterministic mapping $\Pi(\cdot ; \cdot)$ is called a feedback policy if i) ${\Pi}(\cdot ; x)$ is a distribution function for each $x \in \mathbb{R}$; ii) the following SDE (which is the system dynamic after the feedback law ${\Pi}(\cdot ; \cdot)$ is applied)
$$
\d X_{t}=\tilde{b}\left(X_{t},  {\Pi}\left(\cdot; X_{t}\right)\right) \d t+\tilde{\xi}\left(X_{t}^{\Pi}, \Pi\left(\cdot; X_{t}\right)\right) \d W_{t},\;\; t>0 ; \quad X_{0}=x \in \mathbb{R}
$$
has a unique strong solution $\{X_{t}\}_{t \geq 0}$; and iii) the open-loop control $\Pi=\{\Pi_{t}\}_{t \geq 0} \in\mathcal{A}(x)$ where $\Pi_{t}:=\Pi\left(\cdot; X_{t}\right)$. In this case, the resulting open-loop control $\Pi$ is said to be generated from the feedback policy $\Pi(\cdot ;\cdot)$ with respect to the initial state $x$.

On the other hand, for a continuous $h\in \mathcal H$, we have $$\begin{aligned}
\Phi_{h}(\Pi) &=\int_{0}^{1} Q_{\Pi}(1-p) \d h(p)%=-\int_{0}^{1} Q_{\Pi}(t) \d h(1-t)
&%=-\int_{U} u \d h(1-\Pi(u))
=\int_{U} u h^{\prime}(1-\Pi(u))  \d \Pi(u).
\end{aligned}$$

We present the general procedure for solving the optimization problem \eqref{eq:rl},
following \cite{WZZ20a}. Applying the classical Bellman principle of optimality, we %have
%$$
%V(x)=\sup _{\Pi \in \mathcal{A}(x)} \mathbb{E}_x\left[\int_{0}^{s} e^{-\rho t}\left(\tilde{r}\left(X_{t}^{\Pi}, \Pi_{t}\right)+\lambda \Phi_h\left(\Pi_{t}\right)\right)\d t +e^{-\rho s} V\left(X_{s}^{\Pi}\right) \right], ~s>0.
%$$
deduce that the optimal value function $V$ satisfies the Hamilton-Jacobi-Bellman (HJB) equation
\begin{equation}\label{eq:HJB}
\rho v(x)=\max _{\Pi \in \mathcal{M}(U)}\left(\tilde r(x, \Pi)+ \lambda\int_{U} u h^{\prime}(1-\Pi(u))
\d \Pi(u)  +\frac{1}{2} \tilde\xi^{2}(x,\Pi) v^{\prime \prime}(x)+\tilde b(x,\Pi) v^{\prime}(x)\right),
\end{equation}
or equivalently,
$$
\rho v(x)=\max _{\Pi \in \mathcal{M}(U)} \int_{U}\left(r(x, u)+\lambda u h^{\prime}(1-\Pi(u))+\frac{1}{2} \xi^{2}(x,u) v^{\prime \prime}(x)+b(x,u) v^{\prime}(x)\right) \d \Pi(u)  ,
$$
where $v$ denotes the generic unknown solution of the equation. % In the following context, we will focus  the exploratory  problem under  the LQ setting.
The verification theorem then yields  that the following policy
\begin{equation}\label{eq:ver} \Pi^*(x):=\argmax_{\Pi \in \mathcal{M}(U)} \int_{U}\left(r(x, u)+\lambda u h^{\prime}(1-\Pi(u))+\frac{1}{2} \xi^{2}(x,u) v^{\prime \prime}(x)+b(x,u) v^{\prime}(x)\right) \d \Pi(u) \end{equation}
is an optimal policy if it generates an admissible open-loop control for any $x$.

When the regularizer is the entropy,
 \cite{WZZ20a} solve the associated HJB equation explicitly to get the Gibbs (or Boltzmann) measure as the optimal sampler for exploration, which specializes to Gaussian in the LQ case. Solving \eqref{eq:HJB} and/or
 \eqref{eq:ver} generally for Choquet regularizers  remain a (significant)  open question, and in this paper we focus on the LQ setting to
study how different regularizers may generate  the optimal policy distributions.
Specifically, we consider
\begin{equation}\label{eq:abcd}
b(x, u)=A x+B u \quad \text { and } \quad \xi(x, u)=C x+D u,~~ x, u \in \mathbb{R},
\end{equation}
where $A, B, C, D \in \mathbb{R}$, and
\begin{equation}\label{eq:r}
r(x, u)=-\left(\frac{M}{2} x^{2}+R x u+\frac{N}{2} u^{2}+P x+L u\right),~~ x, u \in \mathbb{R},
\end{equation}
where $M \geq 0$, $N>0$,  and $R, P, L \in \mathbb{R}$. Moreover, as in standard LQ theory we assume henceforth that $U=\mathbb{R}$ and thus write $\mathcal{M}=\mathcal{M}(U)$ and $\mathcal{M}^2=\mathcal{M}^2(U)$.
%This type of LQ control problems is quite  important in the classical control literature and we can see later that  it admits elegant and simple solutions. Further,   more complex, nonlinear problems can be approximated by such LQ problems.

Fix an initial state $x\in\R$. For each open-loop control $\Pi\in\mathcal{A}(x)$, denote its mean and variance processes $\{\mu_t\}_{t\geq 0}$ and  $\{\sigma^2_t\}_{t\geq 0}$  by
  $$\mu_t\equiv \mu(\Pi_t)=\int_U u\d \Pi_t( u) ~
  ~~~~~\text{and}~~~~~~\sigma^2_t\equiv \sigma^2(\Pi_t)={\int_U u^2\d \Pi_t(u) -\mu_t^2}.$$
By \eqref{eq:tilde_b} and \eqref{eq:tilde_xi}, we have \begin{equation}\begin{aligned}\label{eq:3}
 \tilde{b}(x, \Pi)= Ax+B\mu(\Pi) ~~~~~\text{and}~~~~~~\tilde{\xi}(x, \Pi)=\sqrt{C^2 x^2+2CDx\mu(\Pi)+D^2[\mu^2(\Pi)+\sigma^2(\Pi)]}.\end{aligned}\end{equation}
Thus,  the state dynamic   $X^{\Pi}$  in \eqref{eq:dyn} is given by  \begin{equation} \label{eq:X_t}\d X_{t}^{\Pi}=(A X_t^{\Pi}+B\mu_t) \d t+\sqrt{(CX_t^{\Pi}+D\mu_t)^2+D^2\sigma_t^2} \, \d W_{t},\quad X_{0}^{\Pi}=x \in \mathbb{R},
 \end{equation} which implies that the state process  only depends on  the mean process $\{\mu_t\}_{t\geq 0}$ and the variance process $\{\sigma^2_t\}_{t\geq 0}$ of a given distributional control $\{\Pi_t\}_{t\geq 0}$.
% Further,  denote
%$$
%H\left(X_{t}^{\Pi}, \Pi_{t}\right):=   \tilde{r}(X_{t}^{\Pi}, \Pi_t)  + \lambda \Phi_h(\Pi_t) .
%$$
Let   $\mathcal{B}$ be the Borel algebra on $\R$. A  control process  $\Pi$ is said to be admissible, denoted by $\Pi\in\mathcal{A}(x)$, if
(i) for each $t \geq 0,$ $\Pi_{t} \in \mathcal{M}$ a.s.;
(ii)  for each $A \in \mathcal{B}$,  $\left\{ \Pi_{t}(A), t \geq 0\right\}$ is $\mathcal{F}_{t}$-progressively measurable;
(iii) for each $t \geq 0, \mathbb{E}[\int_{0}^{t}(\mu_{s}^{2}+\sigma_{s}^{2}) \d s]<\infty$;
(iv) with $\{X_{t}^{\Pi}\}_{t \geq 0}$ solving \eqref{eq:dyn}, $\liminf _{T \rightarrow \infty} e^{-\rho T} \mathbb{E}[(X_{T}^{\Pi})^{2}]=0$;
(v) with $\{X_{t}^{\Pi}\}_{t \geq 0}$  solving \eqref{eq:dyn}, $\mathbb{E}[\int_{0}^{\infty} e^{-\rho t}|\tilde{r}(X_{t}^{\Pi}, \Pi_t)  + \lambda \Phi_h(\Pi_t)| \d t]<\infty$.

In the above, condition (iii) is to ensure that for any $\Pi \in \mathcal{A}(x)$, both the drift and volatility terms of \eqref{eq:dyn} satisfy a global Lipschitz condition and a linear growth condition in the state variable and, hence, the SDE \eqref{eq:dyn} admits a unique strong solution $X^{\Pi}$. Condition (iv) is used to ensure that dynamic programming and verification theorem are applicable, as will be evident in the sequel. Finally, the reward is finite under condition (v).

By \eqref{eq:tilde_r} and  \eqref{eq:r}, we have
\begin{equation}\begin{aligned}\label{eq:2}\tilde r(x,\Pi)=-\frac{M}{2}x^2-Rx \mu(\Pi)-\frac{N}{2} [\mu^2(\Pi)+\sigma^2(\Pi)]-Px-L\mu(\Pi).\end{aligned}\end{equation}
Thus,   plugging \eqref{eq:3} and \eqref{eq:2}  back into \eqref{eq:HJB}, we can derive  the HJB equation for  LQ control as
\begin{equation}\begin{aligned}
\rho v(x)=& \max _{\Pi \in \mathcal{M}^2}\Big\{-Rx\mu(\Pi)-\frac{N}{2}\left[\mu^{2}(\Pi)+\sigma^{2}(\Pi)\right]- L\mu(\Pi)+\lambda \Phi_{h}(\Pi)+ C D x \mu(\Pi) v^{\prime \prime}(x)\\&+\frac{1}{2} D^{2}\left[\mu^{2}(\Pi)+\sigma^{2}(\Pi)\right] v^{\prime \prime}(x)+B \mu(\Pi) v^{\prime}(x)\Big\}+Axv'(x)-\frac{M}{2}x^2-Px+\frac{1}{2}C^2x^2v''(x).
\end{aligned}\label{eq:HJB3}\end{equation}

To analyze and solve this equation, we need to study the maximization problem therein. Denote by $\varphi(x,\Pi)$ the term inside the max operator above. Observe that $\varphi(x,\Pi)$ depends on $\Pi$ via only its mean $\mu(\Pi)$ and variance $\sigma^{2}(\Pi)$, except for the term  $\Phi_{h}(\Pi)$, which motivates us to write
\begin{align}\label{eq:double} \max_{\Pi \in \mathcal{M}^2}\;\varphi(x,\Pi)=\max_{m\in \R, s>0}\;\;\;\max_{\Pi \in \mathcal{M}^2, \mu (\Pi) =m, \sigma^2 (\Pi) = s^2}\;\varphi(x,\Pi).
\end{align}
The inner maximization problem is in turn equivalent to
\begin{align}\label{eq:opt}
%\begin{array}{c}
%\displaystyle
\max_{\Pi \in \mathcal M^2}\Phi_h (\Pi) \mbox{~~~~~subject to~} \mu (\Pi) =m ~\mbox{and}~ \sigma^2 (\Pi) = s^2.
%\end{array}
\end{align}

This is a {\it static} optimization problem, which holds the key to solve the HJB equation \eqref{eq:HJB3} and thus to our exploratory problem with Choquet regularizers. It is interesting to note that when the regularizer is the entropy,  the optimal solution to the above problem is Gaussian, which is indeed the essential reason behind the Gaussian exploration derived in \cite{WZZ20a}. More specifically,
for LQ control any regularized payoff function depends only on the mean and variance processes of the distributional control, and the Gaussian distribution maximizes the entropy  when the mean and variance are fixed. The natural question in our setting is what distribution with given mean and variance maximizes a Choque regularizer, which is exactly the problem (\ref{eq:opt}).
The next section is devoted to solving explicitly this maximization problem (\ref{eq:opt}) of ``mean--variance constrained Choquet regularizers"  with a variety of specific Choquet regularizers.

 \section{Maximizing mean--variance constrained Choquet regularizers}\label{sec:4}

\subsection{General results}

%\cite{WZZ20a} shows that the optimal distribution for entropy regularized exploration in the linear–quadratic (LQ) case
%is Gaussian. This result is essentially due to two facts: 1) for LQ control the entropy regularized payoff function depends only on the mean and variance processes of the distributional control; and 2) the Gaussian distribution maximizes the entropy measure when the mean and variance are fixed.

%As we are working with Choquet regularizers, it is natural to study the
%corresponding second question,  inspired by  \cite{WZZ20a} and  \eqref{eq:HJB3}. %Consider the following optimization problem on $\mathcal{L}^2$.
For given $h\in \mathcal H$, $m\in \R$ and $s>0$, we consider the problem (\ref{eq:opt}),
%Note that the optimization problem  \eqref{eq:opt}  investigates the  optimal  distributions that maximize the regularizers when the  mean and variance  of the distributions are fixed.
which has been motivated by the exploratory control for RL as discussed in the previous section. %It is nonetheless interesting in its own right as a static optimization problem. The optimal distributions derived by  \eqref{eq:opt} maximize the value of the regularizer $\Phi_h$, among all distributions with the given mean and variance.
Note that since $\Phi_h$ is location-invariant and scalable, \eqref{eq:opt} is equivalent to the following problem
\begin{align*}
\label{eq:opt2}
%\begin{array}{c}
%\displaystyle
  s\max_{\Pi \in \mathcal M^2}\Phi_h (\Pi) \mbox{~~~~~subject to~} \mu (\Pi) = 0  ~\mbox{and}~ \sigma^2 (\Pi) = 1.
%\end{array}
\end{align*}
%where $\mu(\Pi)$ is the mean of $\Pi$  and $\sigma^2(\Pi)$  is the variance of $\Pi$.
In what follows, $h'$ represents the right-derivative of $h$, which exists on $[0,1)$ since $h$ is concave on $[0,1]$. %, and $||h'||_2^2=\int _0^1 (h'(t))^2 \d t$.

It turns out  that a general solution to \eqref{eq:opt} has been given by Theorem 3.1 of \cite{LCLW20}. %We will   also use this result to investigate the optimal distribution  functions for the  exploratory  LQ  control problem in Section \ref{sec:6}.
\begin{lemma}\label{lem:liu}
If $h$ is continuous and not constantly zero, then
a maximizer $\Pi^*$ to \eqref{eq:opt} has the following quantile function
\begin{equation}\label{eq:lemma3}
Q_{\Pi^*}(p) =  m + s\frac{ h'(1-p) }{ ||h'||_2}, ~~ \mbox{~a.e. }p\in (0,1),
\end{equation}
and the  maximum value of \eqref{eq:opt} is $\Phi_h(\Pi^*)= s||h'||_2$.
\end{lemma}

In the context of RL, an interesting question arises: Given a distribution used for exploration, what is the regularizer that leads to that distribution? This is a practically important question that can provide interpretability to some widely used samplers for exploration in practice. Theoretically, answering this question is in some sense a converse of Lemma \ref{lem:liu} at least in the LQ setting.

In what follows,   we denote by $\mathcal{M}^2(m,s^2)$    the set of $\Pi \in\mathcal M^2$  satisfying  $\mu(\Pi)=m\in\R$ and $\sigma^2(\Pi)=s^2>0$. Also, recall that given a distribution $\Pi$ the  {\it location-scale family} of   $\Pi$  is the set of all distributions $\Pi_{a,b}$ parameterized by  $a\in\R$ and $b>0$ such that  $\Pi_{a,b}(x)=\Pi((x-a)/b)$  for all $ x \in \mathbb{R}$.
%Instead of searching for an optimal distribution, we now fix a distribution, and search for $\Phi_h$ which it optimizes.
\begin{proposition}\label{prop:opt}
Let  $\Pi\in  \mathcal{M}^2(m,s^2)$ be given, where $m\in \R$ and $s>0$.
Then $\Pi$ maximizes  $\Phi_h$ as well as $\Phi_{\lambda h}$ for any $\lambda>0$ over $\mathcal{M}^2(m,s^2)$
for a continuous $h\in \mathcal H$ specified by
 \begin{equation}\label{eq:opt_h}h'(p)  =  Q_{\Pi}(1-p)-m  ,~~ \mbox{~a.e. }p\in (0,1).\end{equation}
Moreover, any $\hat\Pi$ in the location-scale family of $\Pi$
also maximizes $\Phi_h$ over $ \mathcal M^2(\mu(\hat\Pi),\sigma^2(\hat\Pi))$.

\end{proposition}

\begin{proof} By Lemma \ref{lem:liu}, given a continuous $ h\in\mathcal H$, we have    $$
h'(p)  = \frac{|| h'||_2}{s} (Q_{\Pi}(1-p)-m),~~ \mbox{a.e.} ~ p\in (0,1),$$ where $\Pi$ maximizes  $\Phi_h$ over $\mathcal{M}^2(m,s^2)$.   Since $\Phi_{\lambda h }(\Pi)=\lambda \Phi_h (\Pi)$ for any $\lambda>0$, $\Pi$ that maximizes   $\Phi_{ h }$  also maximizes  $\Phi_{\lambda h }$, which means that a positive constant multiplier in $\Phi_h$ does not affect  problem \eqref{eq:opt}. Hence, $\Pi$  maximizes $\Phi_{ h }$ over $\mathcal{M}^2(m,s^2)$ with $   h'(p)  =  Q_{\Pi}(1-p)-m$ for $p\in(0,1)$ a.e. Moreover, if $\hat\Pi$ is in the location-scale family of $\Pi$,  then we have $\hat\Pi(x)=\Pi( (x-a)/b)$ for some $a\in\R$ and $b>0$ for all $x\in\R$, which implies that $$  h'(p)=  Q_{\Pi}(1-p)-m= ( Q_{\hat\Pi}(1-p)-a)/b-m~~ \mbox{for}~ p\in(0,1) ~a.e.$$ Since $\mu(\hat\Pi)=a+b m$,  it follows that $\hat\Pi$ maximizes  $\Phi_h$ over $ \mathcal M^2(\mu(\hat\Pi),\sigma^2(\hat\Pi)) $.
\end{proof}
A simple but important implication from Proposition \ref{prop:opt} is that  \emph{every non-degenerate distribution with finite first and second moments is the optimizer of some $\Phi_h$ in \eqref{eq:opt} over $\mathcal M^2(m,s^2)$ for some $m\in\R$ and $s>0$}. Therefore, any distribution used for static exploration can be interpreted  by certain suitable Choquet regularizer $\Phi_h$.  Moreover, there is a common distortion function $h$,  which is explicitly specified by Proposition \ref{prop:opt},   for any given  location-scale family,  in the sense that  any distribution function $\Pi$ belonging to this  location-scale family   maximizes $\Phi_h$ over $\mathcal{M}^2(\mu(\Pi),\sigma^2(\Pi))$.
In other words,  a single $\Phi_h$ can serve as 
 the same regularizer for a whole location-scale family of distributions.% commonly maximizes over possibly different parameters of $m\in\R$ and $s>0$.  % Specifically,  for any $m\in\R$ and $s>0$,  define $X'=(X-m)/s$ where   the distribution function of $X$ is given by $\Pi^*$ in Proposition \ref{prop:opt} and let $\Pi'^*$ be the distribution of $X'$. It is clear that $$ h'(p)  =\frac{  Q_{\Pi^*}(1-p)-m}{s}=Q_{\Pi'^*}(1-p)   ,~~ \mbox{~a.e. }p\in (0,1).$$Thus,  $\Pi'^*$  maximizes $\Phi_h (\Pi)$ subject to $\mu (\Pi) =0$ and  $\sigma^2 (\Pi) = 1$  with   $h$ given by   \eqref{eq:prop1}.

\begin{remark}We may also consider optimization of a general functional $I_h$ in which $h$  is not necessarily  concave, such as an inverse S-shaped distortion function.\footnote{A function is called inverse S-shaped (S-shaped) if there exists a point $p^*\in(0,1)$  such that $h$ is concave (convex) on $[0,p^*)$ and is convex (concave) on $(p^*,1]$. Inverse S-shaped distortion functions are common in behavioral decision theory; see e.g. \cite{TK92}.
} This problem is solvable in the setting of \eqref{eq:opt} by replacing $h$ with its concave envelope; see  \cite{PWW20}.   However,
a non-concave $h$ implies that $I_h$ does not satisfy the $\lcx$-consistency (see Theorem \ref{thm:charac}).
%a convex $h$ implies  a convex $I_h(\Pi)$,
This is not desirable  for an exploration regularizer, and hence we will not pursue it in this paper. \end{remark}

In the following subsections,  we  present specific examples applying the above general results, involving several samplers commonly used in RL for exploration, as well as measures commonly used in  finance and operations research  for    evaluating  distribution variability. % Using the results in this subsection, the corresponding optimal distributions  under different Choquet regularizers are derived explicitly.

%  \section{The optimal distributions}\label{sec:5}
%In this section,  we  give several specific measures commonly used in  finance and operations research  when    evaluating  the  distribution variability. Using the results in  Section \ref{sec:4}, the corresponding optimal distributions  under different Choquet regularizers are derived explicitly.

\subsection{Some  common exploratory distributions}
We first present some  examples with   simple distributions  which have been widely used for exploration in the RL literature.
\begin{example}[Bang--bang exploration]\label{ex:bb}
Let  $\Pi$ be a Bernoulli distribution with $\Pi(\{0\})= 1-\epsilon \in (0,1)$ and $\Pi(\{1\})= \epsilon$. In this case, the RL agent explores only two states 0 and 1, which is called a bang--bang exploration.
In particular, in the classical two-armed bandit problem,   0 is the currently more promising arm and 1 is the other arm.
  Proposition \ref{prop:opt} gives $$h' (p) = {\id_{\{p< \epsilon\}} - \epsilon},~~\mbox{~a.e.}~ p\in (0,1),$$
and thus   $h(p)=p\wedge \epsilon - \epsilon p$.  The corresponding regularizer  $\Phi_h$ is given by, using the quantile representation in Lemma \ref{lem:qr},
 $$\begin{aligned}
\Phi_h (\Pi)  &=
 \int_0^{\epsilon}  Q_\Pi(1-p) \d p  - \epsilon  \int_0^1 Q_\Pi(1-p)  \d p= \epsilon (\mu_\epsilon (\Pi)   -  \mu(\Pi)),
 \end{aligned}
$$
where $\mu_\epsilon (\Pi) $   is the $\epsilon$-tail mean defined by
$$
\mu_\epsilon (\Pi):  = \frac 1 \epsilon \int_0^{\epsilon}  Q_\Pi(1-p) \d p .
$$
Since a constant multiplier in $\Phi_h$ does not affect  problem \eqref{eq:opt}, 
  a Bernoulli distribution with parameter $\epsilon$
maximizes
$\Phi _h= \mu_\epsilon - \mu. $
Note that the tail mean corresponds to   ES in risk management with an axiomatic foundation laid out in \cite{WZ21}.
The difference between an ES and the mean,  $ \mu_\epsilon - \mu$,
is a primary example of the  generalized deviation measures in \citet[Example 3]{RUZ06},
which has an axiomatic characterization similar to ES.

\begin{example}[$\epsilon$-greedy exploration]\label{ex:bb'}
Let  $\Pi$ be a  discrete  distribution with $\Pi(\{0\})= 1-\epsilon \in (0,1)$ and $\Pi(\{j\})={\epsilon}/{(2n)}$ for $j\in\{-n,\dots,-1,1,\dots,n\}$. In this case, the RL agent explores  $2n+1$ states  where  $0$ is the currently most ``exploitative" state and $\{-n,\dots, -1, 1, \dots , n\}$ represent the other states surrounding $0$.
 From Proposition \ref{prop:opt}, we have
\begin{equation}\label{eq:dis_unif}
 h' (p) = \sum_{i=1}^n (n-i+1)\id_{\left\{\frac{ (i-1)\epsilon}{2n}\leq p< \frac{ i\epsilon}{2n}\right\}} -\sum_{i=n+1}^{2n} (i-n)\id_{\left\{\frac{ (i-1)\epsilon}{2n}+1-\epsilon\leq p< \frac{ i\epsilon}{2n}+1-\epsilon\right\}},~\mbox{a.e.} ~p\in (0,1);\end{equation}
and thus $h$ is a piece-wise linear function.
%  \begin{equation}\label{eq:dis_unif}\begin{aligned}h(p)=&~\sum_{i=1}^n \left((n-i+1)\left(p-\frac{(i-1)\epsilon}{2n}\right)-\frac{(n-i+1)(n-i+2)\epsilon}{4n}\right)\id_{\left\{\frac{ (i-1)\epsilon}{2n}\leq p< \frac{ i\epsilon}{2n}\right\}}+\frac{(n+1)\epsilon}{4} \\&-\sum_{i=n+1}^{2n}\left( (i-n)\left(p-1+\epsilon-\frac{(i-1)\epsilon}{2n}\right)+\frac{(i-n)(i-n-1)\epsilon}{4n}\right)\id_{\left\{\frac{ (i-1)\epsilon}{2n}+1-\epsilon\leq p< \frac{ i\epsilon}{2n}+1-\epsilon\right\}}.\end{aligned}\end{equation}
 An example of $h$ in \eqref{eq:dis_unif} is plotted in Figure \ref{fig:1}.
 Using the quantile representation in Lemma \ref{lem:qr}, the corresponding regularizer  $\Phi_h$ is given by
\begin{align*}
\Phi_h (\Pi) =\epsilon \left(\sum_{i=1}^n \mu_\epsilon^+ (i, \Pi)-\sum_{i=n+1}^{2n} \mu_\epsilon^- (i, \Pi)\right),
\end{align*}
where $\mu_\epsilon^+ (i,\Pi) $ and   $\mu_\epsilon^- (i,\Pi) $  are  defined by
\begin{equation}\label{eq:mu^+}
\mu_\epsilon^+ (i, \Pi):  = \frac {n-i+1} {\epsilon} \int_{\frac{(i-1)\epsilon}{2n}}^{\frac{ i\epsilon}{2n}}  Q_\Pi(1-p) \d p ~~~~\mbox{for}~~~ i=1,\dots,n,
\end{equation}
and
\begin{equation}\label{eq:mu^-}
\mu_\epsilon^- (i, \Pi):  = \frac {i-n} {\epsilon} \int_{\frac{ (i-1)\epsilon}{2n}+(1-\epsilon)}^{\frac{ i\epsilon}{2n}+(1-\epsilon)}  Q_\Pi(1-p) \d p ~~~~\mbox{for}~~~ i=n+1,\dots,2n.
\end{equation}
\end{example}

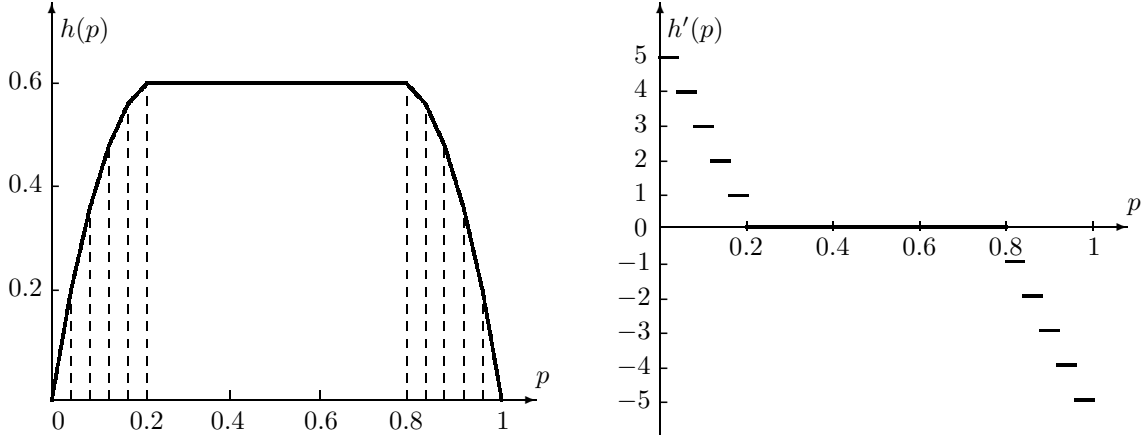
\begin{figure}[t]
{  \setlength{\unitlength}{2.3cm}
\begin{center}
\small
\begin{picture}(2.6,2.4)(0,0)
\put(0,0){\vector(0,1){2.3}}
\put(0,0){\vector(1,0){2.8}}
\put(0,-0.17){$ 0$}

\put(-0.25,1.8){$ 0.6$}
\put(0,1.8) {\small -}
\put(-0.25,1.2){$ 0.4$}
\put(0,1.2) {\small -}
\put(-0.25,0.6){$ 0.2$}
\put(0,0,6) {\small -}
\put(2.8,0.1){$p$}
\put(0.05,2.1){$ h(p)$}
\put(0.45,-0.17){$ 0.2$}
\put(0.92,-0.17){$ 0.4$}
\put(1.46,-0.17){$ 0.6$}
\put(1.95,-0.17){$ 0.8$}
\put(2.57,-0.17){$ 1$}

\linethickness{0.3mm}
\qbezier(0,0)(0.055,0.315)(0.11,0.63)
\qbezier(0.11,0.63)(0.165,0.87)(0.22,1.11)
\qbezier(0.22,1.11)(0.275,1.29)(0.33,1.47)
\qbezier(0.33,1.47)(0.385,1.59)(0.44,1.71)
\qbezier(0.44,1.71)(0.495,1.77)(0.55,1.83)

\linethickness{0.1mm}
\multiput(0.55,0)(0.0,0.1){18}{\line(0,1){0.05}}
\multiput(0.44,0)(0.0,0.1){17}{\line(0,1){0.05}}
\multiput(0.33,0)(0.0,0.1){15}{\line(0,1){0.05}}
\multiput(0.22,0)(0.0,0.1){11}{\line(0,1){0.05}}
\multiput(0.11,0)(0.0,0.1){6}{\line(0,1){0.05}}
\linethickness{0.3mm}
\qbezier(0.55,1.83)(1.3,1.83)(2.05,1.83)
\linethickness{0.3mm}
\qbezier(2.05,1.83)(2.105,1.77)(2.16,1.71)
\qbezier(2.16,1.71)(2.215,1.59)(2.27,1.47)
\qbezier(2.27,1.47)(2.325,1.29)(2.38,1.11)
\qbezier(2.38,1.11)(2.435,0.87)(2.49,0.63)
\qbezier(2.49,0.63)(2.5455,0.315)(2.6,0)
\linethickness{0.1mm}
\multiput(2.05,0)(0.0,0.1){18}{\line(0,1){0.05}}
\multiput(2.16,0)(0.0,0.1){17}{\line(0,1){0.05}}
\multiput(2.27,0)(0.0,0.1){15}{\line(0,1){0.05}}
\multiput(2.38,0)(0.0,0.1){11}{\line(0,1){0.05}}
\multiput(2.49,0)(0.0,0.1){6}{\line(0,1){0.05}}

\multiput(1.55,0)(0.0,0.1){1}{\line(0,1){0.05}}
\multiput(1.03,0)(0.0,0.1){1}{\line(0,1){0.05}}
\end{picture}\quad \quad \quad \quad\quad\quad
\begin{picture}(2.6,2.4)(0,0)
\put(0,-0.2){\vector(0,1){2.5}}
\put(0,1){\vector(1,0){2.7}}
%\put(2.37,-0.17){$ 1$}
\put(-0.15,0.95){$ 0$}
\put(-0.15,1.95){$ 5$}
\put(0,1.95) {\small -}
\put(-0.15,1.75){$ 4$}
\put(0,1.75) {\small -}
\put(-0.15,1.55){$ 3$}
\put(0,1.55) {\small -}
\put(-0.15,1.35){$ 2$}
\put(0,1.35) {\small -}
\put(-0.15,1.15){$ 1$}
\put(0,1.15) {\small -}

\put(2.7,1.1){$ p$}
\put(0.04,2.1){$ h'(p)$}
\put(-0.25,0.75){$ -1$}
\put(0,0.75) {\small -}
\put(-0.25,0.55){$ -2$}
\put(0,0.55) {\small -}
\put(-0.25,0.35){$ -3$}
\put(0,0.35) {\small -}
\put(-0.25,0.15){$ -4$}
\put(0,0.15) {\small -}
\put(-0.25,-0.05){$ -5$}
\put(0,-0.05) {\small -}

\linethickness{0.1mm}
\multiput(2.5,0.98)(0.0,0.1){1}{\line(0,1){0.05}}
\multiput(2,0.98)(0.0,0.1){1}{\line(0,1){0.05}}
\multiput(1.5,0.98)(0.0,0.1){1}{\line(0,1){0.05}}
\multiput(1,0.98)(0.0,0.1){1}{\line(0,1){0.05}}
\multiput(0.5,0.98)(0.0,0.1){1}{\line(0,1){0.05}}
%\put(0.3,-0.17){$ 0.1$}
\put(0.4,0.84){$ 0.2$}
\put(0.9,0.84){$ 0.4$}
\put(1.4,0.84){$ 0.6$}
\put(1.9,0.84){$ 0.8$}
\put(2.47,0.84){$ 1$}

\linethickness{0.3mm}
\qbezier(0,1.98)(0,1.98)(0.1,1.98)
\qbezier(0.1,1.78)(0.2,1.78)(0.2,1.78)
\qbezier(0.2,1.58)(0.3,1.58)(0.3,1.58)
\qbezier(0.3,1.38)(0.4,1.38)(0.4,1.38)
\qbezier(0.4,1.18)(0.4,1.18)(0.5,1.18)
\qbezier(0.5,1)(0.6,1)(2,1)
\qbezier(2,0.8)(2,0.8)(2.1,0.8)
\qbezier(2.1,0.6)(2.1,0.6)(2.2,0.6)
\qbezier(2.2,0.4)(2.2,0.4)(2.3,0.4)
\qbezier(2.3,0.2)(2.4,0.2)(2.4,0.2)
\qbezier(2.4,0)(2.4,0)(2.5,0)
\end{picture}
\end{center}}

\caption{The plots of $h$ (left panel) and $h'$ (right panel) in Example \ref{ex:bb'} corresponding to a discrete distribution $\Pi$  where $n=5$ and $\epsilon=0.4$.}\label{fig:1}
\end{figure}

This example is related  to the $\epsilon$-greedy strategy in multi-armed bandit problem, where $\epsilon$ signifies  the probability of exploring. To be specific,   the $\epsilon$-greedy exploration is to select the current best arm with probability $1-\epsilon$, and the other $2n$ arms uniformly with probability $\epsilon/(2n)$. It is worth noting that  ES is also   used as a  criterion  in the  multi-armed bandit  problem with exploration; see  \cite{CZT20} and \cite{BG21}. \end{example}%\comb{XYZ: This is potentially interesting. Can we use it to explain the epsilon-greedy strategy in multi-armed bandit problem?}\com{Very interesting. I do not know.}

\begin{example}[Exponential exploration] Let %$X$ be a nonnegative random variable with distribution function
$\Pi$ be an exponential distribution with mean $1$. It follows from Proposition \ref{prop:opt} that
$$h' (p) =-\log(p)-1,~~\mbox{~a.e.}~p\in (0,1),$$ and thus $h(p)=-p\log(p)$.  The corresponding Choquet regularizer  $\Phi_h$ is given by
$$
\Phi_h (\Pi) =-
\int_0^1 Q_\Pi(1-p) (\log(p)+1) \d p=:\rm CRE(\Pi),\;\;\Pi \in \mathcal M,
$$
where $${\rm CRE}(\Pi):=-
\int_0^\infty  \Pi([x,\infty)) \log(\Pi([x,\infty)))\d x,$$
which is called the {\it cumulative residual entropy} (CRE) and studied by \cite{RCVW04} and \cite{HC20}. \cite{TSN17} argue that CRE can be viewed as a measure of dispersion or variability. Thus, the exponential exploration can be interpreted by   the CRE regularizer.
  %This definition can be extended to general random variables (unnecessarily nonnegative).
\label{exm:exp}\end{example}
\begin{remark}To provide some parametric flexibility for CRE,
\cite{PN13} introduce the  {\it generalized} cumulative residual entropy (GCRE) given by
$$\operatorname{GCRE}_{n}(\Pi):=\frac{1}{n !} \int_{0}^{\infty} \Pi([x,\infty))\left(-\log(\Pi([x,\infty)))\right)^{n} \mathrm{~d} x, \quad n \in \mathbb{N}_{+}.$$
Clearly, $\operatorname{GCRE}_{1}(\Pi)=\operatorname{CRE}(\Pi)$. From Lemma \ref{lem:qr},    it follows
$$\operatorname{GCRE}_{n}(\Pi)=\int_{0}^{1} Q_\Pi(1-p) \mathrm{d} h_{1, n}(p),$$
where
$$h_{1, n}(p)=\frac{1}{n !}p(-\log(p))^{n}, ~~\mbox{~a.e.}~ p\in(0,1).$$
It is easy to see that $h_{1, n}(0)=h_{1, n}(1)=0$, and
$h_{1, n}^{\prime \prime}(p)=(-\log(p))^{n-2}(\log(p)+n-1)/(p(n-1)!)$
which is negative if and only if $ 0< p<e^{1-n}$. This means that $h_{1, n}$ is not concave on $(0,1)$ for $n>1$. Therefore, $\operatorname{GCRE}_{n}$ is {\it not} a Choquet regularizer  for $n>1$.
\end{remark}

%\begin{remark} If $h(t)=-t\log t$, then the regularizer $\phi_h$ is given by $$
%\Phi_h (\Pi) =-
%\int_0^\infty  \Pi([x,\infty)) \log \Pi([x,\infty) )\d x,
%$$
%which is the  the cumulative residual entropy in  Example \ref{exm:exp}.  Using \eqref{eq:lemma3}, we know that a maximizer $\Pi^*$ satisfies $$Q_{\Pi^*}(t)=m-v(1+\log (1-t)),$$ from which  we obtain $\Pi^*(x)=1-e^{\frac{m-v}{v}-\frac{x}{v}}$. When $m=v=1$, we can  reduce to  the special case in Example \ref{exm:exp}.

%\end{remark}
\begin{example}[Gaussian exploration]
\label{ex:gaussian}
If $\Pi$ is a Gaussian distribution,
then  Proposition \ref{prop:opt} gives
$$h' (p) = z (1-p),~~\mbox{~a.e.}~p\in (0,1),$$
where %$\mathrm N$ is the standard normal distribution function and $\mathrm N^{-1}$ is its quantile function.
$z$ is the quantile function of a standard normal distribution.\footnote{In statistics, the quantile of a standard normal distribution corresponding to a test statistic is often referred to as a z-score -- hence the notation $z$.}
This gives $h(p)=\int_0^p z(1-s)\d s$, which is plotted in Figure \ref{fig:2}. The corresponding regularizer  $\Phi_h$ is given by \begin{equation}\label{eq:nor}
\Phi_h (\Pi) =
\int_0^1 Q_\Pi(1-p)   z  (1-p) \d p =
\int_0^1 Q_\Pi(p)  z (p) \d p,\;\;\Pi \in \mathcal M. \end{equation}
Thus, any Gaussian distribution maximizes the regularize $\Phi_h  $ given by $\Phi_h (\Pi)=\int_0^1 Q_\Pi(p) z (p) \d p$. This example also indicates that there are multiple regularizers (including the above regularizer and differential entropy) that induce Gaussian exploration.
\label{exa:Nor}\end{example}
\begin{figure}[t]
{\setlength{\unitlength}{2.3cm}
\begin{center}
\small
\begin{picture}(2.6,2.4)(0,0)
\put(0,0){\vector(0,1){2.3}}
\put(0,0){\vector(1,0){2.8}}
\put(0,-0.17){$ 0$}

\put(-0.25,1.8){$ 0.4$}
\put(0,1.8) {\small -}
\put(-0.25,1.35){$ 0.3$}
\put(0,1.35) {\small -}
\put(-0.25,0.9){$ 0.2$}
\put(0,0.9) {\small -}
\put(-0.25,0.45){$ 0.1$}
\put(0,0.45) {\small -}
\put(2.8,0.1){$p$}
\put(0.05,2.1){$ h(p)$}
\put(0.43,-0.17){$ 0.2$}
\put(0.95,-0.17){$ 0.4$}
\put(1.47,-0.17){$ 0.6$}
\put(1.99,-0.17){$ 0.8$}
\put(2.57,-0.17){$ 1$}

\linethickness{0.3mm}
\qbezier(0,0)(0.5,1.80)(1.3,1.83)
\qbezier(1.3,1.83)(2.1,1.80)(2.6,0)

\multiput(0.52,0)(0.0,0.1){1}{\line(0,1){0.05}}
\multiput(2.08,0)(0.0,0.1){1}{\line(0,1){0.05}}
\multiput(1.56,0)(0.0,0.1){1}{\line(0,1){0.05}}
\multiput(1.04,0)(0.0,0.1){1}{\line(0,1){0.05}}
\end{picture}\quad \quad \quad \quad\quad\quad
\begin{picture}(2.6,2.4)(0,0)
\put(0,-0.2){\vector(0,1){2.5}}
\put(0,1){\vector(1,0){2.7}}

\put(-0.15,0.95){$ 0$}
%\put(-0.15,1.15){$ 1$}
%\put(0,1.15) {\small -}
\put(-0.15,1.75){$ 4$}
\put(0,1.75) {\small -}
%\put(-0.15,1.55){$ 3$}
%\put(0,1.55) {\small -}
\put(-0.15,1.35){$ 2$}
\put(0,1.35) {\small -}
%\put(-0.15,1.95){$ 5$}
%\put(0,1.95) {\small -}

\put(2.7,1.1){$ p$}
\put(0.04,2.1){$ h'(p)$}
%\put(-0.25,0.75){$ -1$}
%\put(0,0.75) {\small -}
\put(-0.25,0.55){$ -2$}
\put(0,0.55) {\small -}
%\put(-0.25,0.35){$ -3$}
%\put(0,0.35) {\small -}
\put(-0.25,0.15){$ -4$}
\put(0,0.15) {\small -}
%\put(-0.25,-0.05){$ -5$}
%\put(0,-0.05) {\small -}

\linethickness{0.1mm}
\multiput(2.5,0.98)(0.0,0.1){1}{\line(0,1){0.05}}
\multiput(2,0.98)(0.0,0.1){1}{\line(0,1){0.05}}
\multiput(1.5,0.98)(0.0,0.1){1}{\line(0,1){0.05}}
\multiput(1,0.98)(0.0,0.1){1}{\line(0,1){0.05}}
\multiput(0.5,0.98)(0.0,0.1){1}{\line(0,1){0.05}}

\put(0.4,0.84){$ 0.2$}
\put(0.9,0.84){$ 0.4$}
\put(1.4,0.84){$ 0.6$}
\put(1.9,0.84){$ 0.8$}
\put(2.47,0.84){$ 1$}

\linethickness{0.3mm}
\qbezier(0,1.9)(0.01,1.58)(0.04,1.50)
\qbezier(0.04,1.50)(0.2,1.08)(1.25,1)
\qbezier(1.25,1)(2.41,0.82)(2.46,0.4)
\qbezier(2.46,0.4)(2.49,0.32)(2.5,0)
\end{picture}
\end{center}}
\caption{The plots of $h$ (left panel) and $h'$ (right panel) in Example \ref{ex:gaussian} corresponding to a Gaussian distribution.}\label{fig:2}
\end{figure}
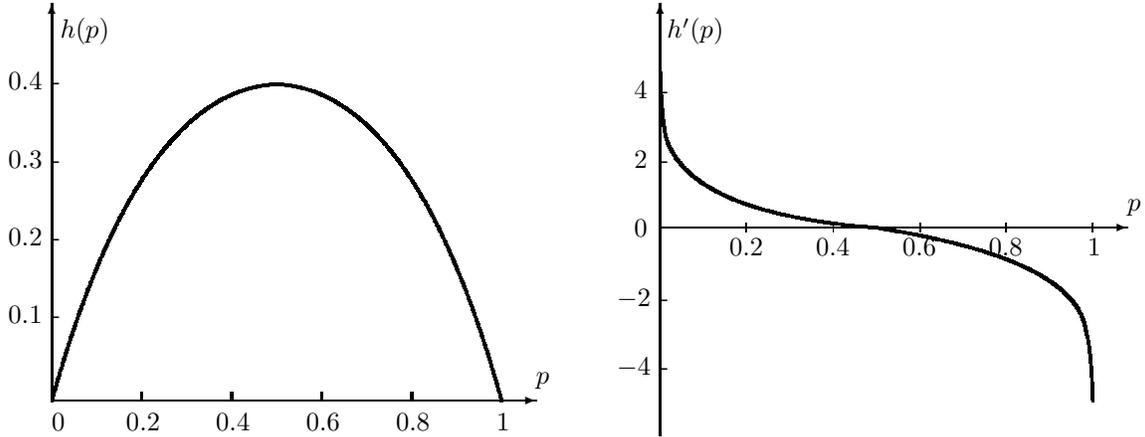

 \subsection{The inter-ES difference as a Choquet regularizer }\label{sec:inter-ES}
We  look at a regularizer  based on ES.  For $\Pi\in \mathcal M$,  ES at level $p$  is defined as %(which corresponds to the tail mean $\mu_{1-t}$ in Example \ref{ex:bb})
$$
\mathrm{ES}_{p}(\Pi):=\frac{1}{1-p} \int_{p}^{1} Q_{\Pi}(r) \mathrm{d} r, \quad  p\in(0,1),
$$
and the left-ES is defined as
$$
\mathrm{ES}_{p}^{-}(\Pi):=\frac{1}{p} \int_{0}^{p} Q_{\mathrm{\Pi}}(r) \mathrm{d} r, \quad   p\in(0,1).
$$
For $\alpha\in(0,1)$,  let
  \begin{equation}\label{eq:IER1}h_{\alpha}(p):= p/(1-\alpha) \wedge 1+(\alpha-p)/(1-\alpha) \wedge 0,\;\;p\in [0,1].\end{equation}   Define $\Phi_{h_{\alpha}}={\mathrm{IER}_\alpha}$  by $$
{\mathrm{IER_\alpha}}(\Pi):=\ES_\alpha(\Pi) -\ES^{-}_{1-\alpha}(\Pi) ,
%= \frac{1}{1-\alpha} \int_{\alpha}^{1} Q_{\Pi}(r) \mathrm{d} r-\frac{1}{1-\alpha} \int_{0}^{1-\alpha} Q_{\mathrm{\Pi}}(r)\mathrm{d} r,
 $$
 which is known as the inter-ES difference. Here, we assume $\alpha\in[1/2,1)$. The inter-ES difference  is a relatively new notion:  it appears in Example 4 of \cite{WWW20c} as a signed Choquet integral.  In a recent work by  \cite{BTWW20}, various   properties  are studied to underline the special role the inter-ES difference plays among other variability measures.

 %We have $h'(t)=\frac{1}{1-\alpha}$ for $t<1-\alpha$ and $h'(t)=-\frac{1}{1-\alpha}$ for $t>\alpha$ and $h'(t)=0$ for  $1-\alpha<t<\alpha$ .

 \begin{proposition}
Suppose that $\alpha\in[1/2,1)$.
For $m\in \R$ and $s^2>0$, the optimization problem
\begin{align*}
\label{eq:IER1}
\max_{\Pi \in \mathcal M^2}{\mathrm{IER}_\alpha} (\Pi) \mbox{~~~~~subject to~} \mu (\Pi) =m ~\mbox{and}~ \sigma^2 (\Pi) = s^2
\end{align*}
is solved by a  three-point distribution $\Pi^*$ with its quantile function uniquely specified as   \begin{equation}\label{eq:QIER2}Q_{\Pi^*}(p)=m+\frac{s}{\sqrt{2(1-\alpha)}}
\left[\id_{\{p>\alpha\}}-\id_{\{p\leq1-\alpha\}}\right],~~\mbox{~a.e.}~p\in (0,1). \end{equation}
\end{proposition}
\begin{proof}
Note that for $\Phi_h={\mathrm{IER}_\alpha}$,  we have   $$h'(p)=\frac{1}{1-\alpha}\id_{\{p<1-\alpha\}}-\frac{1}{1-\alpha}\id_{\{p\geq \alpha\}}$$ for $\alpha\in[1/2,1),$
By \eqref{eq:lemma3}, we can  show that a maximizer $\Pi^*$ satisfies \eqref{eq:QIER2}, which is a three-point distribution.
\end{proof}

So the inter-ES difference regularizer encourages exploration at three points. One of them is the mean $m$ corresponding to the best single-point exploitation without exploration, while the other two spots are symmetric to $m$ capturing the exploration part. %This generalizes the $\varepsilon$-greedy strategy.

\begin{remark} For $\alpha\in[1/2,1)$,  if we  take the function $h_{\alpha}(p)= \id_{[1-\alpha,\alpha]}(p)$, $p\in[0,1],$  the inter-quantile difference $\Phi_{h_{\alpha}}:={\mathrm{IQR}_\alpha}$  is given by $$
{\mathrm{IQR_\alpha}}(\Pi):=Q^+_{\Pi}(\alpha)-Q_{\Pi}(1-\alpha),
 $$
which  is a classical measure of statistical dispersion widely used  in e.g., box plots.
Unlike the inter-ES difference,  the distortion function $h_{\alpha}$ for ${\mathrm{IQR}_{\alpha}}$ is not concave.  However,  the concave envelopes of $h$ is give by $h^*(p)= p/(1-\alpha) \wedge 1+(\alpha-p)/(1-\alpha) \wedge 0$,   $p\in [0,1]$, which is exactly \eqref{eq:IER1}. According to Theorem 1 in \cite{PWW20}, we have $\sup_{\Pi\in\mathcal M^2}\rm IQR_{\alpha}(\Pi)=\sup_{\Pi\in\mathcal M^2}\rm IER_\alpha(\Pi)$
 and the  maximizer is obtained   by  $\Pi^*$ which satisfies \eqref{eq:QIER2}. Thus, the optimization problem is still solvable even if $h$ is not concave.

%Similarly, let function $h(t)=(t-1+q)_+/{(q-p)} \wedge 1, t \in[0,1]$, where $0<p<q<1 ;$ the regularizer $\Phi_{\RVaR}:=\Phi_h$ is given as  the Range-Value-at-Risk (RVaR):  $$ \Phi_{\RVaR_{p,q}}(\Pi)=\frac{1}{q-p}\left((1-p)\int_{p}^{1} Q_{\Pi}(r) \mathrm{d} r-(1-q) \int_{q}^{1} Q_{\Pi}(r) \mathrm{d} r\right), ~~~~\Pi\in \mathcal{M},  $$ see for example \cite{CDS10} and \cite{ELW18}.  Unlike the inter-ES difference,  the distortion function $h$ for $\Phi_{\RVaR_{p,q}}$ is not concave.  However, we can still apply Proposition \ref{prop:2} to obtain the  maximizer $\Pi^*$ given by   $$Q_{\Pi^*}(t)=m+v(q-p)^{-\frac{3}{2}}\one_{\{1-q\leq t\leq1-p\}},$$ which is a two-point distribution on $\{m,m+v(q-p)^{-\frac{3}{2}}\}$.
\end{remark}

\subsection{The $L^1$-Wasserstein distance to Dirac measures as a Choquet regularizer}\label{sec:Wass}
 Let $W:\M\times\M \to \R_+$ be a statistical distance between two distributions, such as a Wasserstein distance.
 Since an exploration is essentially to move away from degenerate (Dirac) distributions,
a natural way to encourage  exploration  is to use $W(\Pi,\delta_x)$, where $\delta_x$ is the Dirac measure at $x\in \R$, as a regularizer. Moreover, to remove the location dependence, we modify the regularizer to be  $\min_{x\in \R} W(\Pi,\delta_x)$.
For any  statistical distance satisfying $W(\Pi,\hat\Pi)=0$ if and only if $\Pi=\hat\Pi$, it is clear that
$\min_{x\in \R} W(\Pi,\delta_x)=0$ if and only if $\Pi$ itself is a Dirac measure (hence deterministic).
%We therefore prefer a higher value of $\min_{x\in \R} W(\Pi,\delta_x)$ to reflect exploration.

The use of Wasserstein distance
to model distributional uncertainty in other settings naturally gives rise to a regularization term, yielding a theoretical justification for its use in practice; see for example  \cite{PW07}, \cite{EK17} and \cite{BCZ21} that  formulate different  models with distributional robustness based on Wasserstein distances.

 We focus on the case where $W$ is the Wasserstein $L^1$ distance, defined as
 $$W_1(\Pi, \hat\Pi):= \int_0^1  | Q_\Pi (p) - Q_{\hat\Pi}(p)| \d p.$$
 In this case, $W_1(\Pi, \delta_x)$ is the $L^1$ distance between $x$ and $X\sim \Pi$, and it is well known via $L^1$ loss minimization that the minimizers of
 $\min_{x\in \R} W_1(\Pi,\delta_x)$ are the medians of $\Pi$ (unique if $Q_\Pi$ is continuous):
  $$
  \argmin_{x\in \R} W_1(\Pi,\delta_x)  =[Q_\Pi(1/2), Q^+_\Pi(1/2)].
 $$
 Moreover,  for a median of $\Pi$, $x^*\in [Q_\Pi(1/2), Q^+_\Pi(1/2)]$,  we have that $W_1(\Pi, \delta_{x^*})$ is the mean-median deviation; namely
\begin{align*}
 \min_{x\in \R} W_1(\Pi,\delta_x)   & = W_1(\Pi, \delta_{x^*}) \\ & =    \int_0^{1/2} (x^*-Q_\Pi(p))\d p+ \int_{1/2}^1 (Q_\Pi(p) -x^*)\d p  \\ & =   \int_{1/2}^1  Q_\Pi(p) \d p - \int_0^{1/2}  Q_\Pi(p)\d p.
 \end{align*}
This in turn shows that $  \argmin_{x\in \R} W_1(\Pi,\delta_x) $ belongs to the class of Choquet regularizers.

 \begin{proposition}\label{prop:5}
For $m\in \R$ and $s^2>0$, the optimization problem
$$\begin{aligned}
\label{eq:WD}
\max_{\Pi \in \mathcal M^2}\min_{x\in \R} W_1(\Pi,\delta_x) \mbox{~~~~~subject to~} \mu (\Pi) =m ~\mbox{and}~ \sigma^2 (\Pi) = s^2,
\end{aligned}
$$
is solved by a unique $\Pi^*$ with the quantile function specified as \begin{equation}\label{eq:bbsol}
	Q_{\Pi^*}(p) = m+s\id_{\{p> 1/2\}}-s\id_{\{p\le 1/2\}} , ~~ a.e.~p\in (0,1).
	\end{equation}
\end{proposition}
\begin{proof}
 Applying Lemma \ref{lem:qr} to get $\min_{x\in \R} W_1(\Pi,\delta_x)=\Phi_h(\Pi)$ with $h'(p)=1$ for $p< 1/2$ and $h'(p)=-1$ for $p\ge1/2$.
 Using  \eqref{eq:lemma3} in Lemma \ref{lem:liu} yields \eqref{eq:bbsol},
which implies a symmetric two-point distribution.   \end{proof}

As    $\Phi_h(\Pi)=\min_{x\in \R} W_1(\Pi,\delta_x)$  induces a symmetric exploration around the mean, we call it  a symmetric Wasserstein regularizer  with $h(p)=p \id_{\{p<1/2\}} +(1-p)\id_{\{p\ge1/2\}}$.
Next, let us discuss two-point asymmetric exploration. Suppose that two directions are not symmetric, and we would like to regularize in a way to encourage more exploration in a certain direction. Take a constant  $\alpha\in(0,1)$, and choose $W$ as an asymmetric Wasserstein distance %(it suffices to consider $W(\Pi,\delta_x)$ here)
 $$
 W_1^{\alpha}(\Pi, \hat\Pi)= \int_0^1  \big(\alpha( Q_\Pi (p) - Q_{\hat\Pi}(p))_+ + (1-\alpha)( Q_\Pi (p) - Q_{\hat\Pi}(p))_-\big)  \d p.$$
The corresponding minimizers are the $\alpha$-quantiles
  $$
  \argmin_{x\in \R} W^\alpha_1(\Pi,\delta_x)  =[Q_\Pi(\alpha), Q^+_\Pi(\alpha)],
 $$
  and  for $x^*\in [Q_\Pi(\alpha), Q_\Pi(\alpha)]$,  we have
\begin{align*}
 \min_{x\in \R} W^\alpha_1(\Pi,\delta_x)   & = W^\alpha_1(\Pi, \delta_{x^*}) \\ & =    \int_0^{\alpha} (1-\alpha) (x^*-Q_\Pi(p))\d p+ \int_{\alpha}^1\alpha (Q_\Pi(p) -x^*)\d p  \\ & = \alpha \int_{\alpha}^1  Q_\Pi(p)
 \d p   -   (1-\alpha)   \int_0^{\alpha} Q_\Pi(p) \d p.
 \end{align*}
 We call   $\Phi_h(\Pi)=\min_{x\in \R} W^\alpha_1(\Pi,\delta_x)$   an asymmetric Wasserstein regularizer  with $h(p)=\alpha p \id_{\{p< 1-\alpha\}} +(1-\alpha)(1-p)\id_{\{p\ge 1-\alpha\}}$.
 \begin{proposition}

For $m\in \R$ and $s^2>0$, the optimization problem
\begin{align*}
\label{eq:WD}
\max_{\Pi \in \mathcal M^2}\min_{x\in \R} W^\alpha_1(\Pi,\delta_x) \mbox{~~~~~subject to~} \mu (\Pi) =m  ~\mbox{and}~ \sigma^2 (\Pi) = s^2
\end{align*}
has a unique maximizer $\Pi^*$  with the quantile function uniquely specified as \begin{equation}\label{eq:asW}
Q_{\Pi^*}(p) = m+s\left(\frac{ \alpha }{1-\alpha}\right)^{1/2}\id_{\{p> \alpha\}}-s\left(\frac{ 1-\alpha }{\alpha}\right)^{1/2} \id_{\{p\leq\alpha\}} , ~~\mbox{~a.e.}~ p\in (0,1).
 \end{equation}
\end{proposition}

 \begin{proof}
For  $\Phi_h(\Pi)=\min_{x\in \R} W^\alpha_1(\Pi,\delta_x)$, we have
$$ \begin{aligned}
 h'(p)=\alpha ~ ~\text{for}~~ p<1-\alpha ~~\text{and}~ ~h'(p)=-1+\alpha~~ \text{for}~~ p\ge1-\alpha.  \label{eq:bbh}
 \end{aligned}$$
 Using \eqref{eq:lemma3}, the optimization problem   has a solution $\Pi^*$ satisfying \eqref{eq:asW},
   which is an asymmetric two-point distribution. %The special case of $\alpha=1/2$ corresponds to \eqref{eq:bbsol}.
 \end{proof}
% \com{I think we could use this section (and perhaps the next one as well) as the motivation for introducing general $\Phi_h$ as an regularizer used for regularization.}

To recap, the Wasserstein $L^1$ regularization encourages possibly  asymmetric (with respect to the mean)  two-point exploration, which is an instance of the bang-bang exploration in Example \ref{ex:bb}.
  \subsection{The Gini mean difference or maxiance as a Choquet regularizer}\label{sec:Gini}

% \com{Added new section, December 07, 2020}
By letting $h(p)= p-p^2$, $p\in [0,1]$, we consider the regularizer $\Phi_\sigma:=\Phi_h$ given by
 $$
\Phi_\sigma(\Pi)=   \int_\R  \Big(  \Pi([x,\infty)) -  \Pi^2 ([x,\infty)) \Big)\d x.
 $$
 There are two ways to represent  $ \Phi_\sigma(\Pi)  $ in terms of  two iid copies $X_1$ and $X_2$ from the distribution $\Pi$.
 First, $\Phi_\sigma$ can be rewritten as  $$
  \Phi_\sigma(\Pi) = \frac 12 \E[|X_1-X_2|],
 $$
 which is the \emph{Gini mean difference} (e.g., \citealp{FWZ17}; sometimes without the factor  $1/2$).
 Alternatively,  $\Phi_\sigma$ can be  represnted  as
 $$
 \Phi_\sigma(\Pi) = \E[ \max\{X_1,X_2\}]-\mu(\Pi),
 $$
 which is called  the \emph{maxiance}   by \cite{EL20}.
 The two representations are identical as seen from the following equality \begin{align*}
  \E[ \max\{X_1,X_2\}]- \mu(\Pi)  & = \E\left[ \max\{X_1,X_2\}-  \frac12(X_1+X_2)\right]
  \\&=  \E\left[ \max\{X_1,X_2\}-\frac12 \left({\max\{X_1,X_2\}}+ \min\{X_1,X_2\} \right)\right]
  \\&= \frac12   \E\left[ \max\{X_1,X_2\}- \min\{X_1,X_2\} \right]  = \frac 12 \E[|X_1-X_2|].
 \end{align*}
As argued by \cite{EL20},  the  maxiance can be seen as the dual version of  the variance, due to the following identities
$$
\sigma^2(\Pi) = \int_\R (x-\mu(\Pi))^2 \d \Pi ,~~~~\Phi_\sigma(\Pi) = \int_\R (x-\mu(\Pi)) \d \Pi^2.
$$
Moreover, the  maxiance can be used to approximate a local index of absolute risk aversion in \cite{Y87}'s dual theory of choice under risk, which is similar to the role of  variance in the classic expected utility theory.

We now show that the maxiance regularizer  $\Phi_\sigma$ leads to a uniform distribution for exploration.
\begin{proposition}\label{prop:uni}
For $m\in \R$ and $s^2>0$, the optimization problem
\begin{align}
\label{eq:opt3}
\max_{\Pi \in \mathcal M^2}\Phi_\sigma (\Pi) \mbox{~~~~~subject to~} \mu (\Pi) =m ~\mbox{and}~ \sigma^2 (\Pi) = s^2
\end{align}
has a unique maximizer $\Pi^*=\mathrm{U}[m-\sqrt 3s,m+\sqrt 3s]$.
\end{proposition}
\begin{proof}
Note that for $\Phi_h=\Phi_\sigma$,  we have $h'(p)=1-2p$. It follows from \eqref{eq:lemma3} that a maximizer $\Pi^*$ is a uniform distribution. By matching the moments in \eqref{eq:opt3}, we obtain $\Pi^*=\mathrm{U}[m-\sqrt 3s,m+\sqrt 3s]$.
The uniqueness statement is guaranteed by e.g. Theorem 2 of \cite{PWW20}.
\end{proof}

Proposition \ref{prop:uni} provides a foundation for a uniformly distributed exploration strategy on $\R$. Note that this is different from the result of uniform distributions maximizing entropy   on a fixed, given  bounded region: here in our setting the region is {\it not} fixed, since we allow $\Pi$ to be chosen from arbitrary distributions on $\R$, and thus the bounded region $[m-\sqrt 3s,m+\sqrt 3s]$ is endogenously derived rather than exogenously  given.
%Moreover, $m$ and $s^2$ may change over time and state in the dynamic setting. In the RL context, the uniform distribution on a pre-specified  bounded region is not optimal in general for entropy regularized problems.

\begin{remark}
%Let $\sigma(\Pi)$ be the standard deviation of $\Pi\in \mathcal M^2$.
The inequality
$$
\sigma(\Pi) \ge \sqrt{3}\Phi_\sigma(\Pi) \mbox{~for all $\Pi \in \mathcal M^2$}
$$
 is known as Glasser's inequality (\citealt{G62}).
For the uniform distribution $\Pi^* $ in Proposition \ref{prop:uni} with $\sigma(\Pi^*)=s$,  we have
$\Phi_\sigma(\Pi^*) =\sqrt{3}s/3$ by Lemma \ref{lem:liu}. Thus, $\Pi^*$ attains the sharp bound of Glasser's inequality, which holds naturally since $\Pi^*$ maximizes $\Phi_\sigma$ for a fixed $\sigma^2$.
\end{remark}

\section{Solving the exploratory stochastic LQ control problem}\label{sec:6}
We are now ready to solve the exploratory stochastic LQ control  problem  presented in Section \ref{sec:3}.
Let \begin{equation}\label{eq:W} W(x,\Pi)=
 \E_x \left[ \int_0^\infty e^{-\rho t} \left( \tilde{r}(X_{t}^{\Pi}, \Pi_t)  + \lambda \Phi_h(\Pi_t)
\right)\d t \right], \;\;x\in \mathbb{R},~ \Pi \in \mathcal A(x).
\end{equation}   We have the following result based on Lemma \ref{lem:liu}.
 \begin{proposition}\label{prop:LQ_prob}
Let a continuous $h\in \mathcal H$ be given.
 For any  $\Pi=\{\Pi_t\}_{t\ge 0}\in \mathcal A(x)$ with mean process $\{\mu_t\}_{t\ge 0}$ and variance process $\{\sigma_t^2\}_{t\ge 0}$, there exists  $\Pi^*=\{\Pi^*_t\}_{t\ge 0}\in \mathcal A(x)$ given by
 \begin{equation}\label{eq:general}
Q_{\Pi^*_t}(p) = \mu_t + \sigma_t\frac{ h'(1-p) }{ ||h'||_2}, ~~ ~~\mbox{~a.e.}~p\in (0,1),\;\;t\geq0,
\end{equation}
 which has the same mean  and variance processes satisfying
 $W(x,\Pi^*)\ge W(x,\Pi)$.

 \end{proposition}
\begin{proof}
It follows from  \eqref{eq:X_t} and \eqref{eq:2} that the term $ \E_x \left[ \int_0^\infty e^{-\rho t}  \tilde{r}(X_t^\Pi, \Pi_t)
\d t \right]$ in \eqref{eq:W}   only depends on  the mean process $\{\mu_t\}_{t\ge 0}$ and  variance process $\{\sigma_t^2\}_{t\ge 0}$ of $\{\Pi_t\}_{t\ge 0}$. Thus,
for any fixed $t\geq 0$, choose $\Pi_t^*$ with mean $\mu_t$ and variance $\sigma_t^2$
that maximizes $\Phi_h(\Pi)$. Form  Lemma \ref{lem:liu}, it follows that $\Pi^*_t$  satisfies \eqref{eq:general}
and the  maximum value  is $\Phi_h(\Pi_t)= \sigma_t||h'||_2$. Clearly, 
the strategy $\Pi^*=\{\Pi^*_t\}_{t\ge 0}\in \mathcal A(x)$ is the desired one.  %Thus, the LQ problem in \eqref{eq:rl}  can be  maximized within  a location--scale family of distributions.
\end{proof}
%\begin{remark} We know from  Remark \ref{rem:general} that if  both the reward term and the dynamic process only depend on the mean process $\mu_t$ and the $
%a$-th moment  process $\sigma^{a}_t$  of  $\Pi_t$ for $ t\geq 0$, then we have $W\left(x, \Pi^{*}\right) \geqslant W(x, \Pi)$
%with  $\Pi^{*}_t$   satisfying   $$
%Q_{\Pi_t^*}(p) =  \mu_t + \sigma_t \frac{\left|h^{\prime}(1-p)-c_{h, b}\right|^{b}}{h^{\prime}(1-p)-c_{h, b}}[h]_{b}^{1-b}, ~\mbox{ if }~h^{\prime}(1-p)-c_{h, b} \neq 0, ~ \mbox{ and }Q_{\Pi_t^*}(p) =  \mu_t~  \mbox{otherwise}.
%$$    \end{remark}

Proposition \ref{prop:LQ_prob} indicates that the control  problem  \eqref{eq:rl} in the LQ setting  %with mean and variance constraints
 is  maximized within  a location--scale family of distributions, which is determined only by $h$.
 
We go back to the HJB equation  \eqref{eq:HJB3}.  It follows from \eqref{eq:double}--\eqref{eq:opt} along with Lemma \ref{lem:liu} that \eqref{eq:HJB3} is equivalent to
\begin{equation}\label{eq:HJB2}\begin{aligned}
\rho v(x)=& \max _{\mu\in\R, \sigma>0}\Big[-Rx\mu-\frac{N}{2}\left(\mu^{2}+\sigma^{2}\right)-L\mu+\lambda \sigma\left\|h'\right\|_2+ C D x \mu v^{\prime \prime}(x)+\frac{1}{2} D^{2}\left(\mu^{2}+\sigma^{2}\right) v^{\prime \prime}(x)\\&+B \mu v^{\prime}(x)\Big]+Axv'(x)-\frac{M}{2}x^2-Px+\frac{1}{2}C^2x^2v''(x).\end{aligned}\end{equation}
Applying the first-order conditions, we  get  the maximizers
$$\mu^*(x)=\frac{CDxv''(x)+Bv'(x)-Rx-L}{N-D^2v''(x)} ~~\text{and}~~ (\sigma^*( x))^2=\frac{\lambda^2 \left\|h'\right\|^2_2}{(N-D^2v''(x))^2}$$
of the max operator in \eqref{eq:HJB2}, which in turn leads to the optimal
distributional policy $\Pi^*(\cdot;x)$ prescribed by Lemma \ref{lem:liu}.

Bringing the above expressions of $\mu^*(x)$ and $\sigma^*(x)$ back into \eqref{eq:HJB2}, we can further write  the HJB equation as
\begin{equation}\label{eq:5}\rho v(x)=\frac{[CDxv''(x)+Bv'(x)-Rx-L]^2+\lambda^2\left\|h'\right\|^2_2}{2[N-D^2v''(x)]}+\frac{1}{2}\left[C^2v''(x)-M\right]x^2+[Av'(x)-P]x.\end{equation}
We now solve this equation explicitly. Denote
$$\Delta=[\rho-(2 A+C^{2})] N+2(B+CD)R-D^{2} M.$$
 Under the assumptions that $\rho>2A+C^2$ and $MN>R^2$, a smooth solution to \eqref{eq:5} is given by
$$
v(x)=\frac{1}{2} k_{2} x^{2}+k_1x+k_{0},
$$
where\footnote{Values of $k_2$, $k_1$ and $k_0$ are obtained by solving the system of  equations
$\rho k_{2}=\frac{(k_{2}(B+C D)-R)^{2}}{N-k_{2} D^{2}}+k_{2}\left(2 A+C^{2}\right)-M,$
$\rho k_{1}=\frac{\left(k_{1} B-L\right)\left(k_{2}(B+C D)-R\right)}{N-k_{2} D^{2}}+k_{1} A-P$,
and
$\rho k_0=\frac{(k_1B-L)^2+\lambda^2\|h'\|_2^2}{2(N-k_2D^2)}.$}
\begin{equation}\label{eq:k2}
k_{2}= \frac{\Delta-\sqrt{\Delta^{2}-4[(B+C D)^{2}+(\rho-(2 A+C^{2})) D^{2}](R^2-M N)}}{2[(B+C D)^{2}+D^{2}(\rho-(2 A+C^{2})) ]} ,
\end{equation}
 \begin{equation}\label{eq:k1}
k_{1}=\frac{P\left(N-k_{2} D^{2}\right)-L R}{k_{2} B(B+C D)+(A-\rho)\left(N-k_{2} D^{2}\right)-B R},	
\end{equation} and \begin{equation}\label{eq:k0}k_0=\frac{(k_1B-L)^2+\lambda^2\|h'\|_2^2}{2\rho(N-D^2k_2)}.\end{equation}

We can verify easily that $k_{2}<0$. Hence, $v$ is concave, a property that is essential for $v$ to be actually the value function.
Next, we state  the main result of this section, whose proof  follows essentially the same lines of that of Theorem 4 in \cite{WZZ20a}, thanks to the analysis above and the results obtained. We omit the  details here.
\begin{theorem}\label{thm:3}%Suppose that the reward function is given by
%$$
%r(x, u)=-\left(\frac{M}{2}x^2+Rx u +\frac{N}{2} u^2+Px+L u\right)
%$$
%with  $M\geq0$, $N>0,$ and  $R, P, L\in \R$.
Consider the LQ control specified by \eqref{eq:abcd}--\eqref{eq:r}, where we assume
$M\geq0$, $N>0,$ $MN>R^2$ and \footnote{The constraint  on  $\rho$  is used not  only  to ensure $k_2<0$ but also to show $\liminf _{T \rightarrow \infty} e^{-\rho T} \mathbb{E} [ (X_{T}^{\Pi} )^{2} ]=0$; see  the proof of Theorem 4 in \cite{WZZ20a} for more details.}   $$\rho>2 A+C^{2}+\max \left(\frac{D^{2} R^{2}-2 N R(B+C D)}{N}, 0\right).$$
Then the value function in \eqref{eq:rl} is given by
$$
V(x)=\frac{1}{2} k_{2} x^{2}+k_1x+k_{0}, \quad x \in \mathbb{R},
$$
where $k_{2}$, $k_1$  and $k_{0}$ are as in \eqref{eq:k2}-\eqref{eq:k0}, respectively.  The  optimal feedback  policy has the distribution  function $\Pi^*(\cdot;x)$  whose quantile function is
\begin{equation}\label{eq:opt_pi}
Q_{\Pi^*(\cdot;x)}(p) = \frac{(k_2(B+CD)-R)x+k_1B-L}{N-k_2D^2} + \frac{\lambda h'(1-p)}{N-k_2D^2}, ~~\mbox{~a.e.}~ p\in (0,1),\quad x \in \mathbb{R},\end{equation} with the mean and variance  given by  \begin{equation}\label{eq:mu_sigma}\mu^*(x)=\frac{(k_2(B+CD)-R)x+k_1B-L}{N-k_2D^2} ~~\text{and}~~ (\sigma^*(x))^2=\frac{\lambda^2 \left\|h'\right\|^2_2}{(N-k_2D^2)^2},\quad x \in \mathbb{R}.\end{equation}
Finally, the associated optimal state process $\left\{X_{t}^{*}\right\}_{t \geq 0}$ with $X_{0}^{*}=x$ under $\Pi^{*}(\cdot ; \cdot)$ is the unique solution of the SDE
\begin{align*}
\d X_{t}^{*}&=\left[\left(A+\frac{B\left(k_{2}(B+C D)-R\right)}{N-k_{2} D^{2}}\right) X_{t}^{*}+\frac{B\left(k_{1} B-L\right)}{N-k_{2} D^{2}}\right] \d t \\
&\quad+\sqrt{\left[\left(C+\frac{D\left(k_{2}(B+C D)-R\right)}{N-k_{2} D^{2}}\right) X_{t}^{*}+\frac{D\left(k_{1} B-L\right)}{N-k_{2} D^{2}}\right]^{2}+\frac{D^2\lambda^2 \left\|h'\right\|^2_2 D^2}{(N-k_{2} )^2}} \d W_{t}.
\end{align*}
\end{theorem}

Some remarks are in order. First of all,   \eqref{eq:opt_pi}  implies that
for any  Choquet regularizer,  the optimal exploratory distribution in the  regularized LQ problem   is  uniquely determined   by $h'$. Note that $h'(x)$ is the ``probability weight" put on $x$ when calculating the (nonlinear) Choquet expectation; see e.g.  \cite{Q82} and \cite{GS89}. Second, we can see from \eqref{eq:mu_sigma} that the mean of the optimal  distribution does not depend on the exploration represented by  $h$ and  $\lambda$, and only the variance does. In particular, the mean is exactly the same as the one in \cite{WZZ20a} when the differential  entropy is used as a regularizer, which is also identical to the optimal control of the classical, non-exploratory LQ problem. Third, the mean of the exploration distributions  is a linear function of  the state, while its variance is independent of the state.

 These observations are  intuitive in  the context of RL.  Different $h$'s  correspond to different Choquet regularizers; hence they will certainly   affect the way and the level of exploration.  Also, the more weight put on the level of exploration, the more spreaded out the exploration becomes around the current position.
Furthermore, the second and third observations above show a perfect separation between exploitation and exploration, as the former is captured by the mean and the latter by the variance of the optimal  exploration distributions.  This property is also consistent with the LQ case studied in  \cite{WZZ20a} and \cite{WZ20} even though a different type of regularizer is applied therein.

Next,  we  investigate  optimal exploration samplers under the LQ framework for some  concrete choices of $h$ studied in Section \ref{sec:4}. For convenience, we  denote $$~~ \tilde\sigma^*(x):=\frac{\sigma^*(x)}{\|h'\|_2}\equiv \frac{\lambda }{N-D^2k_2}.$$
Theorem \ref{thm:3} yields that the mean of the  optimal distribution is independent of $h$;  so we will specify only  its quantile function and variance  for each  $h$ discussed below.  Recall that  the expressions of $\mu^*(x)$ and $(\sigma^*(x))^2$ for  a general  $h$ are given by \eqref{eq:mu_sigma}.
\begin{itemize}\item[(i)]    Let $h(p)=(p\wedge \epsilon-\epsilon p)$, leading to  $\Phi_h(\Pi)=\epsilon (\mu_\epsilon (\Pi)   -  \mu(\Pi))$; see  Example \ref{ex:bb}.
 The optimal policy is $\epsilon$-greedy,  given as
$$\Pi^*\left(\{\mu^*(x)+(1-\epsilon) \tilde\sigma^*(x)\}\right)\equiv \Pi^*\left(\left\{\frac{(k_2(B+CD)-R)x+k_1B-L+(1-\epsilon)\lambda}{N-k_2D^2}\right\} \right)= \epsilon, $$
and $$ \Pi^*\left(\{\mu^*(x)-\epsilon\tilde\sigma^*(x)\}\right)\equiv \Pi^*\left(\left\{\frac{(k_2(B+CD)-R)x+k_1B-L-\epsilon\lambda}{N-k_2D^2}\right\} \right)= 1-\epsilon.$$
At each state $x$, the control policy takes a more ``promising" action at $\mu^*(x)-\epsilon\tilde\sigma^*(x)$ with a large probability $1-\epsilon$, and tries  an alternative action $\mu^*(x)+(1-\epsilon) \tilde\sigma^*(x)$ with probability $\epsilon$.\footnote{Precisely speaking, the policy presented here is
not exactly the $\epsilon$-greedy strategy in the classical two-arm bandit problem because the two ``arms" in our setting depend on the current state $x$ and hence
are dynamically changing over time. However, at any point of time one needs to explore only two action points.}
Since $\|h'\|^2_2=\epsilon(1-\epsilon)$, the variance of $\Pi^*$ is $$(\sigma^*(x))^2=\frac{\epsilon(1-\epsilon)\lambda^2 }{(N-k_2D^2)^2}.$$
\item[(ii)] Let  $h(p)$ be specified by the discrete exploration in \eqref{eq:dis_unif}, leading to  \begin{align*}
\Phi_h (\Pi) =\epsilon\left(\sum_{i=1}^n \mu_\epsilon^+ (i, \Pi)-\sum_{i=n+1}^{2n} \mu_\epsilon^- (i, \Pi)\right),
\end{align*}
where $\mu_\epsilon^+ (i,\Pi) $ and   $\mu_\epsilon^- (i,\Pi) $  are  defined by \eqref{eq:mu^+}
and \eqref{eq:mu^-}; see Example \ref{ex:bb'}.
The  optimal  policy is  a $(2n+1)$-point   distribution  given as
$$\Pi^*\left(\{\mu^*(x)+j \tilde\sigma^*(x)\}\right)\equiv \Pi^*\left(\left\{\frac{(k_2(B+CD)-R)x+k_1B-L+j\lambda}{N-k_2D^2}\right\} \right)= \frac{\epsilon}{2n}, $$ for $j\in\{-n,\dots,-1,1,\dots,n\}$,  and $$\Pi^*\left(\{\mu^*(x)\}\right)\equiv\Pi^*\left(\left\{\frac{(k_2(B+CD)-R)x+k_1B-L}{N-k_2D^2}\right\}\right)=1- \epsilon.$$
Similarly,  at each state $x$, the control policy takes a more ``exploitative" action at $\mu^*(x)$ with a large probability $1-\epsilon$, and tries  $2n$ alternative actions $\mu^*(x)+j \tilde\sigma^*(x)$ for $j\in\{-n,\dots,-1, 1,\dots,n\}$,   each with probability $\epsilon/(2n)$. Since  $\|h'\|^2_2=\epsilon(n+1)(2n+1)/6$,   the variance of $\Pi^*$ is given by $$(\sigma^*(x))^2=\frac{\epsilon(n+1)(2n+1)\lambda^2 }{6(N-k_2D^2)^2}.$$

\item[(iii)] Let $h(p)=-p\log(p)$, corresponding to
   $\Phi_h (\Pi)=\int_0^\infty  \Pi([x,\infty)) \log(\Pi([x,\infty)))\d x$;  see  Example \ref{exm:exp}.   The  optimal  policy is   a shifted-exponential  distribution   given as
$$\Pi^*(u; x)=1-\exp\left\{\frac 1\lambda \left[(k_2(B+CD)-R)x+k_1B-L\right]-1\right\}\exp\left\{-\frac{1}{\lambda}(N-D^2k_2)u\right\}.$$ Since  $\|h'\|^2_2=1$, the variance of $\Pi^*$ is given by $$(\sigma^*(x))^2=\frac{\lambda^2 }{(N-k_2D^2)^2}.$$
\item[(iv)] Let $h(p)=\int_0^p z(1-s)\d s$ where $z$ is the standard normal quantile function. We have  $\Phi_h (\Pi) =\int_0^1 Q_\Pi(p) z (p) \d p$; see Example \ref{exa:Nor}.    The  optimal  policy  is a normal   distribution  given by
$$
{\Pi}^{*}(\cdot ; x)=  {\mathrm N}\left( \frac{(k_2(B+CD)-R)x+k_1B-L}{N-k_2D^2}, \frac{\lambda^2 }{(N-k_2D^2)^2}\right),
$$  owing to the fact that $\|h'\|^2_2=1$. Recall that the optimal distribution is also Gaussian  in \cite{WZZ20a} using the entropy regularizer. This is an example of different regularizers leading to the same class of exploration samplers. On the other hand, examining more closely  the Gaussian policy derived above and the one 
 in \cite[eq. (40)]{WZZ20a}, we observe that the means of the two are identical but the variance of the former is the square of that of the latter. The reason of the discrepency in variance is because the maximized mean--variance constrained Choquet regularizer $\Phi_h(\Pi)$ is always linear in the given standard deviation $\sigma$  whereas the corresponding maximized entropy regularizer $\mbox{DE}(\Pi)$ is logorithmic  in $\sigma$.

\item[(v)] Let $h(p)= p/(1-\alpha) \wedge 1+(\alpha-p)/(1-\alpha) \wedge 0$ with $\alpha\in[1/2,1)$. Then $\Phi_h(\Pi)=\ES_\alpha(\Pi) -\ES^{-}_{1-\alpha}(\Pi)$; see Section \ref{sec:inter-ES}.   The  optimal  policy is  a three-point   distribution  given as
$$\Pi^*\left(\left\{\frac{(1-\alpha)[(k_2(B+CD)-R)x+k_1B-L]+\lambda}{(1-\alpha)(N-k_2D^2)}\right\}\right)= 1-\alpha, $$ $$\Pi^*\left(\left\{\frac{(k_2(B+CD)-R)x+k_1B-L}{N-k_2D^2}\right\}\right)=2\alpha-1,$$ and $$\Pi^*\left(\left\{\frac{(1-\alpha)[(k_2(B+CD)-R)x+k_1B-L]-\lambda}{(1-\alpha)(N-k_2D^2)}\right\}\right)=1- \alpha.$$
Since  $\|h'\|^2_2=2a/(1-\alpha)^2$,   the variance of $\Pi^*$ is given by $$(\sigma^*(x))^2=\frac{2\alpha\lambda^2 }{(1-\alpha)^2(N-k_2D^2)^2}.$$
\item[(vi)] Let $h(p)=\alpha p \id_{\{p< 1-\alpha\}}  +(1-\alpha)(1-p)\id_{\{p\ge 1-\alpha\}}$, with $\Phi_h(\Pi)=\min_{x\in \R} W_1(\Pi,\delta_x)$; see Section \ref{sec:Wass}.  The  optimal  feedback policy  is an asymmetric
two-point distribution  given as
$$\Pi^*\left(\left\{\frac{(k_2(B+CD)-R)x+k_1B-L+\alpha\lambda}{N-k_2D^2}\right\}\right)= 1-\alpha,$$  and$$\Pi^*\left(\left\{\frac{(k_2(B+CD)-R)x+k_1B-L-(1-\alpha)\lambda}{N-k_2D^2}\right\}\right)= \alpha.$$ Since  $\|h'\|^2_2=\alpha(1-\alpha)$, the variance of $\Pi^*$ is given by $$(\sigma^*(x))^2=\frac{\alpha(1-\alpha)\lambda^2 }{(N-k_2D^2)^2}.$$
\item[(vii)]
Let $h(p)=p-p^2$. Then  $
  \Phi_h(\Pi) =  \E[|X_1-X_2|]/2$; see Section \ref{sec:Gini}.  The  optimal  policy is  a uniform  distribution   given as
$$\Pi^*(\cdot;x)=\mathrm{U}\left[\frac{(k_2(B+CD)-R)x+k_1B-L-\lambda}{N-k_2D^2},\frac{(k_2(B+CD)-R)x+k_1B-L+\lambda}{N-k_2D^2}\right].$$
Since $\|h'\|^2_2=1/3$, the variance of $\Pi^*$ is given by $$(\sigma^*(x))^2=\frac{\lambda^2 }{3(N-k_2D^2)^2}.$$

Note here the uniform distribution is on a state-dependent bounded region centering around the mean $\mu^*(x)$,  rather than on a pre-specified  bounded region. 

\end{itemize}

\section{Conclusion}
\label{sec:7}

This paper develops a framework for continuous-time RL that can generate or indeed interpret/explain many broadly practiced  distributions for exploration. The main contributions are conceptual/theoretical rather than algorithmic: Theorem \ref{thm:3} does not lead directly to an algorithm to compute optimal policies, because the expression \eqref{eq:opt_pi} involves the model parameters which are unknown in the RL context. That said, our results do provide important guidance for devising RL algorithms. On one hand, Theorem \ref{thm:3} may imply  a provable policy improvement theorem and hence result in a q-learning theory analogous to that in the entropy-regularized setting recently established by  \cite{JZ22}. On the other hand,
the explicit form \eqref{eq:opt_pi} can suggest special structure of  function approximators for learning optimal distributions and thereby greatly reduce the number of parameters needed for function approximation.

Another conceptual contribution  of the paper is that
%To quantify the exploration level of the distributions,    we  focus  on the  distortion risk metrics whose distortion functions are   concave,  which include many popular risk and variability  measures commonly used in the literature of finance, optimization, and risk management. The main contributions  are conceptual rather than algorithmic:  We
it establishes a link between risk metrics  and RL.
%As   exploration  is encouraged,   we derive the optimal distributions via maximizing the regularizers with mean and variance constraints and   show that   our  regularizers are commonly maximized by  a location-scale family of distributions.
%The   LQ stochastic control problem  is solved in the RL context, and we show that  every location-scale family of distributions is optimal for some Choquet regularizers.  Some examples are  presented to show  the practice of the regularizers that can be selected as a  regularization to quantify  the information gain of exploring. Thus,  in addition to DE, our Choquet regularizers can provide a new regularization choice to balance the exploration-exploitation in  RL.
This paper is the first to do so, and  the attempt is by no means comprehensive. The rich literature on decision theory and risk metrics is expected to further bring in  new insights and directions into the RL study, not only related to regularization, but also in terms of motivating new objective functions and axiomatic approaches for  learning.

The theory developed in this paper opens up several research directions. Here we comment on some. One is to develop the corresponding q-learning theory mentioned earlier.
Another is to find the ``best Choquet regularizer" in terms of  efficiency of the resulting RL algorithms. Yet another problem  is in financial application: to formulate a continuous-time mean--variance portfolio selection problem with a Choquet regularizer and compare the
performance with its entropy counterpart solved in \cite{WZ20}.

Last but not least, the Choquet regularizers   proposed in this paper are defined for distributions on $\R$, while many RL applications involve multi-dimensional action spaces.  Because Choquet regularizers are characterized by quantile additivity as in Theorem \ref{thm:charac} while  quantile functions are not well defined for distributions on $\R^d$ with $d>1$, it is very challenging  to study Choquet regularizers in high dimensions. To overcome the difficulty,  the first possible attempt  is to minic \eqref{eq:def2} by defining, for  distributions $\Pi$ on $\R^d$, the functional
 $$
\Phi^{\rm joint}_h(\Pi)= \int_{\R^d} h\circ \Pi([\mathbf x,\infty))\d \mathbf x.
 $$
 This formulation requires some further conditions on $h\in \mathcal H$ to guarantee desirable properties, and it is unclear whether we can derive the corresponding optimizers in a form similar to Proposition \ref{prop:opt}.
 Another possible idea is to use
 $$
 \Phi^{\rm sum}_h(\Pi) =\sum_{i=1}^d   \int_{\R } h\circ \Pi_i([  x,\infty))\d   x \mbox{~~or~~}  \Phi^{\rm prod}_h(\Pi) =\prod_{i=1}^d    \int_{\R } h\circ \Pi_i([  x,\infty))\d   x,
 $$
 where $\Pi_i$ is the $i$-th marginal distribution of $\Pi$.
This formulation relies only on the marginal distributions of $\Pi$, allowing us to utilize the existing  results  for Choquet regularizers  on $\R$. Either formulation mentioned above requires a thorough analysis in a future study.

\subsubsection*{Acknowledgements}
Wang is supported by the Natural Sciences and Engineering Research Council of Canada (RGPIN-2018-03823, RGPAS-2018-522590). Zhou is supported  by a start-up grant and the Nie Center
for Intelligent Asset Management at Columbia University.

\end{document}